\documentclass[letterpaper]{article}
\usepackage{uai2019}
\usepackage[margin=1in]{geometry}

\usepackage{times}

\title{Instructions for Authors}

\author{
    Yang Song\thanks{~Joint first authors. Correspondence to Yang Song \textless{}yangsong@cs.stanford.edu\textgreater{} and Stefano Ermon \textless{}ermon@cs.stanford.edu\textgreater{}.}\\
   Stanford University \\
   \\
  \And
   Sahaj Garg$^*$\\
  Stanford University \\
  \And
  Jiaxin Shi \\
  Tsinghua University \\
  \And
  Stefano Ermon\\
  Stanford University
}

\usepackage[utf8]{inputenc} %
\usepackage[T1]{fontenc}    %
\usepackage[hidelinks]{hyperref}       %
\usepackage{url}            %
\usepackage{booktabs}       %
\usepackage{amsfonts}       %
\usepackage{nicefrac}       %
\usepackage{microtype}      %
\usepackage{algorithmicx,algpseudocode,algorithm}

\usepackage{comment}
\usepackage{enumitem} 
\usepackage{bm}
\usepackage{placeins}

\usepackage{amsmath,amsfonts,bm}

\def\1{\bm{1}}

\def\rvy{{\mathbf{y}}}
\def\rvz{{\mathbf{z}}}

\def\vb{{\bm{b}}}

\def\evb{{b}}

\def\mG{{\bm{G}}}

\def\mI{{\bm{I}}}

\DeclareMathAlphabet{\mathsfit}{\encodingdefault}{\sfdefault}{m}{sl}
\SetMathAlphabet{\mathsfit}{bold}{\encodingdefault}{\sfdefault}{bx}{n}

\def\emG{{G}}

\DeclareMathOperator*{\argmin}{arg\,min}

\newcommand{\ptilde}{\tilde{p}}
\newcommand{\cH}{\mathcal{H}}
\newcommand{\tp}{^\mathsf{T}}

\usepackage{wrapfig,tikz}
\usepackage{amsmath,longtable,fancyhdr,booktabs,multirow,graphicx,float}
\usepackage{adjustbox, bigstrut, tabularx, multirow, makecell, diagbox}
\usepackage{amssymb,xcolor,amsthm}
\usepackage{subfigure}
\usepackage{color}
\usepackage{colortbl}
\usepackage{theoremref}
\definecolor{smcolor}{rgb}{0.5490196078431373, 0.33725490196078434, 0.29411764705882354}
\definecolor{cpcolor}{rgb}{0.8392156862745098, 0.15294117647058825, 0.1568627450980392}
\definecolor{dsmcolor}{rgb}{0.17254901960784313, 0.6274509803921569, 0.17254901960784313}
\definecolor{ssmcolor}{rgb}{1.0, 0.4980392156862745, 0.054901960784313725}
\definecolor{ssmvrcolor}{rgb}{0.12156862745098039, 0.4666666666666667, 0.7058823529411765}
\definecolor{bpcolor}{rgb}{0.5803921568627451, 0.403921568627451, 0.7411764705882353}
\definecolor{mlecolor}{rgb}{0.8901960784313725, 0.4666666666666667, 0.7607843137254902}
\definecolor{elbocolor}{rgb}{0.4980392156862745, 0.4980392156862745, 0.4980392156862745}
\definecolor{steincolor}{rgb}{0.7372549019607844, 0.7411764705882353, 0.13333333333333333}
\definecolor{spectralcolor}{rgb}{0.09019607843137255, 0.7450980392156863, 0.8117647058823529}
\def\SM{{\color{smcolor}SM}}
\def\SSM{{\color{ssmcolor}SSM}}
\def\DSM{{\color{dsmcolor}DSM}}
\def\SSMVR{{\color{ssmvrcolor}SSM-VR}}
\def\CP{{\color{cpcolor}CP}}
\def\aBP{{\color{bpcolor}approx BP}}
\def\ABP{{\color{bpcolor}Approx BP}}
\def\Stein{{\color{steincolor}Stein}}
\def\Spectral{{\color{spectralcolor}Spectral}}
\def\MLE{{\color{mlecolor}MLE}}
\def\ELBO{{\color{elbocolor}ELBO}}

\newcommand{\mbf}[1]{\mathbf{#1}}
\newcommand{\bs}[1]{\boldsymbol{#1}}
\newcommand{\mbb}[1]{\mathbb{#1}}
\newcommand{\ud}{\mathrm{d}}
\newcommand{\up}{\mathrm}

\newcommand{\mcal}{\mathcal}
\newcommand{\norm}[1]{\left\lVert#1\right\rVert}

\newtheorem{lemma}{Lemma}
\newtheorem{theorem}{Theorem}
\newtheorem{corollary}{Corollary}
\newtheorem{remark}{Remark}

\newtheorem{proposition}{Proposition}
\newtheorem{assumption}{Assumption}

\newtheorem{innercustomthm}{Theorem}
\newenvironment{customthm}[1]
{\renewcommand\theinnercustomthm{#1}\innercustomthm}
{\endinnercustomthm}

\newcommand{\be}{\begin{equation}}
	\newcommand{\ee}{\end{equation}}

\definecolor{Gray}{gray}{0.85}
\definecolor{LightCyan}{rgb}{0.88,1,1}

\newcolumntype{a}{>{\columncolor{Gray}}c}
\newcolumntype{b}{>{\columncolor{white}}c}

\makeatletter
\usepackage{xspace}
\def\@onedot{\ifx\@let@token.\else.\null\fi\xspace}
\DeclareRobustCommand\onedot{\futurelet\@let@token\@onedot}

\newcommand{\figref}[1]{Fig\onedot~\ref{#1}}
\newcommand{\algoref}[1]{Alg\onedot~\ref{#1}}

\newcommand{\secref}[1]{Section~\ref{#1}}
\newcommand{\tabref}[1]{Tab\onedot~\ref{#1}}
\newcommand{\thmref}[1]{Theorem~\ref{#1}}

\newcommand{\appref}[1]{Appendix~\ref{#1}}

\newcommand{\corref}[1]{Corollary~\ref{#1}}
\newcommand{\lemref}[1]{Lemma~\ref{#1}}
\newcommand{\propref}[1]{Proposition~\ref{#1}}

\newcommand{\assref}[1]{Assumption~\ref{#1}}
\newcommand{\bfx}{\mathbf{x}}
\newcommand{\bfv}{\mathbf{v}}
\newcommand{\bfz}{\mathbf{z}}
\newcommand{\bfe}{{\bs{\epsilon}}}
\newcommand{\bftheta}{{\boldsymbol{\theta}}}
\newcommand{\bfalpha}{{\boldsymbol{\alpha}}}
\newcommand{\bfphi}{{\boldsymbol{\phi}}}

\newcommand{\bfs}{\mathbf{s}}
\newcommand{\bfh}{\mathbf{h}}

\def\ie{\emph{i.e}\onedot}

\def\cf{\emph{cf}\onedot}

\def\wrt{w.r.t\onedot}

\def\aka{a.k.a\onedot}
\def\iid{i.i.d\onedot}

\usepackage[textsize=tiny]{todonotes}

\title{Sliced Score Matching: A Scalable Approach to \\Density and Score Estimation}%

\usepackage{natbib}
\begin{document}
\maketitle

\begin{abstract}
    Score matching is a popular method for estimating unnormalized statistical models. %
    However, it has been so far limited to simple, shallow models or low-dimensional data, due to the difficulty of computing the Hessian 
    of log-density functions. We show this difficulty can be mitigated by 
    projecting the scores onto random vectors before comparing them. This objective, called sliced score matching, only involves Hessian-vector products, which can be easily implemented using reverse-mode automatic differentiation. Therefore, sliced score matching is amenable to more complex models and higher dimensional data compared to score matching. Theoretically, we prove the consistency and asymptotic normality of sliced score matching estimators.
    Moreover, we demonstrate that sliced score matching can be used to learn deep score estimators for implicit distributions. 
    In our experiments, we show sliced score matching can learn deep energy-based models effectively, and can produce accurate score estimates for applications such as variational inference with implicit distributions and training Wasserstein Auto-Encoders. %
\end{abstract}
\section{INTRODUCTION}

Score matching~\citep{hyvarinen2005estimation} is 
particularly suitable for learning unnormalized statistical models, such as energy based ones. It is based on minimizing the distance between the derivatives of the log-density functions (\aka, \emph{score}s) of the data and model distributions.
Unlike maximum likelihood estimation (MLE), the objective of score matching only depends on the scores, which are oblivious to the (usually) intractable partition functions. 
However, score matching 
requires the computation of the diagonal elements of the Hessian of the model's log-density function. This Hessian trace computation is generally expensive~\citep{martens2012estimating}, requiring a number of forward and backward propagations proportional to the data dimension. This severely limits its applicability 
to complex models
parameterized by deep neural networks, such as deep energy-based models~\citep{lecun2006tutorial,wenliang2018}.

Several approaches have been proposed to alleviate this difficulty: \citet{kingma2010regularized} propose approximate backpropagation for computing the trace of the Hessian; \citet{martens2012estimating} develop curvature propagation, a fast stochastic estimator for the trace in score matching; and \citet{vincent2011connection} %
transforms score matching to a denoising problem which avoids second-order derivatives.
These methods have achieved some success, but may suffer from one or more of the following problems: inconsistent parameter estimation, large estimation variance, and cumbersome implementation.

To alleviate these problems, %
we propose sliced score matching, a %
variant of score matching that can scale to deep unnormalized models and high dimensional data. The key intuition is that instead of directly matching %
the high-dimensional scores, %
we match their projections along random directions. %
Theoretically, we show that under some regularity conditions, sliced score matching %
is a well-defined statistical estimation criterion %
that yields consistent and asymptotically normal parameter %
estimates. Moreover, compared to the methods of \citet{kingma2010regularized} and \citet{martens2012estimating}, whose implementations require customized backpropagation for deep networks, sliced score matching only involves Hessian-vector products, thus can be easily and efficiently implemented in frameworks %
such as TensorFlow~\citep{abadi2016tensorflow} and PyTorch~\citep{adam2017automatic}.

Beyond training unnormalized models, sliced score matching can also be naturally adapted as an objective for estimating the score function of a data generating distribution~\citep{Sasaki14score,strathmann2015gradient} by training a score function model parameterized by deep neural networks.
This observation enables many new applications of sliced score matching. For example, we show that it can be used to provide accurate score estimates needed for variational inference with implicit distributions~\citep{huszar2017variational} and learning Wasserstein Auto-Encoders (WAE, \citet{tolstikhin2018wasserstein}). 

Finally, we evaluate the performance of sliced score matching on learning unnormalized statistical models (density estimation) and estimating score functions of a data generating process (score estimation). For density estimation, experiments on deep kernel exponential families~\citep{wenliang2018} and NICE flow models~\citep{dinh14nice} show that our method is either more scalable or more accurate than existing score matching variants.

For score estimation, our method improves the performance of variational auto-encoders (VAE) with implicit encoders, and can train WAEs without a discriminator or MMD loss by directly optimizing the KL divergence between aggregated posteriors and the prior. In both situations we outperformed kernel-based score estimators~\citep{li2018gradient,shi2018spectral} by achieving better test likelihoods and better sample quality in image generation.

\section{BACKGROUND}

Given \iid samples $\{\bfx_1, \bfx_2, \cdots, \bfx_N\} \subset \mbb{R}^D$ from a data distribution $p_d(\bfx)$, our task is to learn an unnormalized density, $\tilde{p}_m(\bfx;\bftheta)$, where $\bftheta$ is from some parameter space $\Theta$. The model's partition function is denoted as $Z_\bftheta$, which is assumed to be existent but intractable. Let $p_m(\bfx; \bftheta)$ be the normalized density determined by our model, we have
\begin{align*}
    p_m(\bfx;\bftheta) = \frac{\tilde{p}_m(\bfx;\bftheta)}{Z_\bftheta}, \quad Z_\bftheta = \int \tilde{p}_m(\bfx;\bftheta)\ud \bfx.
\end{align*}
For convenience, we denote the score functions of $p_m$ and $p_d$ as $\bfs_m(\bfx; \bftheta) \triangleq \nabla_\bfx \log p_m(\bfx; \bftheta)$ and $\bfs_d(\bfx) \triangleq \nabla_\bfx \log p_d(\bfx)$ respectively. Note that since $\log p_m(\bfx;\bftheta) = \log \tilde{p}_m(\bfx; \bftheta) - \log Z_\bftheta$, we immediately conclude that $\bfs_m(\bfx; \bftheta)$ does not depend on the intractable partition function $Z_\bftheta$.

\subsection{SCORE MATCHING}\label{sec:sm}
Learning unnormalized models with maximum likelihood estimation (MLE) can be difficult due to the intractable partition function $Z_\bftheta$. To avoid this, score matching~\citep{hyvarinen2005estimation} minimizes the Fisher divergence between $p_d$ and $p_m(\cdot, \bftheta)$, which is defined as
\begin{align}
    L(\bftheta) \triangleq \frac{1}{2}\mbb{E}_{p_d}[\norm{\bfs_m(\bfx; \bftheta) - \bfs_d(\bfx)}^2_2].\label{eqn:fisher}
\end{align}
Since $\bfs_m(\bfx; \bftheta)$ does not involve $Z_\bftheta$, the Fisher divergence does not depend on the intractable partition function. However, Eq.~\eqref{eqn:fisher} is still not readily usable for learning unnormalized models, as we only have samples and do not have access to the score function of the data $\bfs_d(\bfx)$.

By applying integration by parts, \citet{hyvarinen2005estimation} shows that $L(\bftheta)$ can be written as $L(\bftheta) = J(\bftheta) + \up{C}$ (\cf, Theorem~1 in \citet{hyvarinen2005estimation}), where
\begin{gather}
 J(\bftheta) \triangleq \mbb{E}_{p_d}\bigg[\operatorname{tr}(\nabla_{\bfx} \bfs_m(\bfx; \bftheta)) + \frac{1}{2}\norm{\bfs_m(\bfx; \bftheta)}_2^2\bigg],\label{eqn:sm}
\end{gather}
$\up{C}$ is a constant that does not depend on $\bftheta$, $\operatorname{tr}(\cdot)$ denotes the trace of a matrix, and 
\begin{align}
    \nabla_{\bfx} \bfs_m(\bfx; \bftheta) = \nabla_\bfx^2 \log \tilde{p}_m(\bfx; \bftheta)
\end{align}
is the Hessian of the log-density function. 
The constant can be ignored and the following unbiased estimator 
of the remaining terms 
is used to train $\tilde{p}_m(\bfx; \bftheta)$:
\begin{align*}
    \resizebox{\columnwidth}{!}{$\displaystyle
    \hat{J}(\bftheta; \bfx_1^N) \triangleq %
    \frac{1}{N} \sum_{i=1}^N \bigg[\operatorname{tr}(\nabla_\bfx \bfs_m(\bfx_i; \bftheta)) + \frac{1}{2}\norm{\bfs_m(\bfx_i; \bftheta)}_2^2\bigg],$}\vspace{-1em}
\end{align*}
where $\bfx_1^N$ is a shorthand used throughout the paper for a collection of $N$ data points $\{\bfx_1, \bfx_2, \cdots, \bfx_N\}$ sampled i.i.d. from $p_d$, and $\nabla_\bfx \bfs_m(\bfx_i; \bftheta)$ denotes the Hessian of $\log \tilde{p}_m(\bfx;\bftheta)$ evaluted at $\bfx_i$.

\paragraph{Computational Difficulty.} While the score matching objective $\hat{J}(\bftheta; \bfx_1^N)$ avoids the computation of $Z_\bftheta$ for unnormalized models, it introduces a new computational difficulty: computing the trace of the Hessian of a log-density function, $\nabla_\bfx^2 \log \tilde{p}_m$. A na\"{i}ve approach of computing the trace of the Hessian requires $D$ times more backward passes than computing the gradient $\bfs_m = \nabla_\bfx \log \tilde{p}_m$ (see \algoref{alg:sm} in the appendix). For example, the trace could be computed by applying backpropogation $D$ times to $\bfs_m$ to get each diagonal term of $\nabla_\bfx^2 \log \tilde{p}_m$ sequentially. In practice, $D$ can be many thousands, which can render score matching too slow for practical purposes. Moreover, \citet{martens2012estimating} argues from a theoretical perspective that it is unlikely that there exists an algorithm for computing the diagonal of the Hessian defined by an arbitrary computation graph within a constant number of forward and backward passes.

\subsection{SCORE ESTIMATION FOR IMPLICIT DISTRIBUTIONS}
\label{sec:bg-score-est}
Besides parameter estimation in unnormalized models, score matching can also be used to estimate scores of \emph{implicit distributions}, which are distributions that have a tractable sampling process but without a tractable density. For example, the distribution of random samples from the generator of a GAN~\citep{goodfellow2014generative} is an implicit distribution. Implicit distributions can arise in many more situations such as the marginal distribution of a non-conjugate model~\citep{sun2018functional}, and models defined by complex simulation processes~\citep{tran2017hierarchical}. In many cases learning and inference become intractable due to the need of optimizing an objective that involves the intractable density.

In these cases, directly estimating the score function $\bfs_q(\bfx) = \nabla_\bfx \log q_\bftheta(\bfx)$ based on \iid samples from an implicit %
density $q_\bftheta(\bfx)$ can be useful. For example, suppose our learning problem involves optimizing the entropy $H(q_\bftheta(\bfx))$ of an implicit distribution. This situation is common when dealing with variational free energies~\citep{kingma2013autoencoding}. Suppose $\bfx\sim q_{\bftheta}$ can be reparameterized as $\bfx = g_{\bftheta}(\bfe)$, where $\bfe$ is a simple random variable independent of $\bftheta$, such as a standard normal, and $g_\bftheta$ is a deterministic mapping. We can write the gradient of the entropy with respect to $\bftheta$ as
\begin{align*}
    \nabla_\bftheta H(q_\bftheta) &\triangleq -\nabla_\bftheta \mbb{E}_{q_\bftheta(\bfx)}[\log q_\bftheta(\bfx)]\\
    &= -\nabla_\bftheta \mbb{E}_{p(\bfe)}[\log q_\bftheta(g_\bftheta(\bfe))]\\
    &= -\mbb{E}_{p(\bfe)}[\nabla_\bfx \log q_\bftheta(\bfx)|_{\bfx = g_\bftheta(\bfe)} \nabla_\bftheta g_\bftheta(\bfe)],
\end{align*}
where $\nabla_\bftheta g_\bftheta(\bfe)$ is usually easy to compute. The score $\nabla_\bfx \log q_\bftheta(\bfx)$ is intractable but can be approximated by score estimation. 

Score matching is an attractive solution for score estimation since \eqref{eqn:fisher} naturally serves as an objective to measure the difference between the a trainable score function and the score of a data generating process. We will discuss this in more detail in \secref{sec:sse} and mention some other approaches of score estimation in \secref{sec:score_est}.

\section{DENSITY AND SCORE ESTIMATION WITH SLICED SCORE MATCHING}

\subsection{SLICED SCORE MATCHING}\label{sec:ssm}
We observe that one dimensional problems are usually much easier to solve than high dimensional ones. Inspired by the idea of Sliced Wasserstein distance~\citep{10.1007/978-3-642-24785-9_37}, we consider projecting $\bfs_d(\bfx)$ and $\bfs_m(\bfx; \bftheta)$ onto some random direction $\bfv$ and propose to compare their average difference along that random direction. More specifically, we consider the following objective as a replacement of the Fisher divergence $L(\bftheta)$ in Eq. \eqref{eqn:fisher}: 
\begin{align}
    L(\bftheta; p_\bfv) \triangleq \frac{1}{2}  \mbb{E}_{p_\bfv} \mbb{E}_{p_d}\left[ (\bfv^\intercal \bfs_m (\bfx; \bftheta) - \bfv^\intercal \bfs_d(\bfx))^2 \right].\hspace{-.5em} \label{eqn:s_fisher}
\end{align}
Here $\bfv \sim p_\bfv$ and $\bfx \sim p_d$ are independent, and we require $\mbb{E}_{p_\bfv}[\bfv \bfv^\intercal] \succ 0$ and $\mbb{E}_{p_\bfv}[\|\bfv\|_2^2] < \infty$. Many examples of $p_\bfv$ satisfy these requirements. For instance, $p_\bfv$ can be a multivariate standard normal 
($\mcal{N}(0, I_{D})$), a multivariate Rademacher distribution (the uniform distribution over $\{\pm 1\}^{D}$), or a uniform distribution on the hypersphere $\mbb{S}^{D-1}$ (recall that $D$ refers to the dimension of $\bfx$).

To eliminate the dependence of $L(\bftheta; p_\bfv)$ on $\bfs_d(\bfx)$, we use integration by parts as in score matching (\cf, the equivalence between Eq.~\eqref{eqn:fisher} and \eqref{eqn:sm}). Defining
\begin{multline}
    J(\bftheta; p_\bfv) \triangleq  
    \mbb{E}_{p_\bfv} \mbb{E}_{p_d}\big[\bfv^\intercal \nabla_\bfx \bfs_m(\bfx; \bftheta) \bfv \\ + 
    \frac{1}{2} \left(\bfv^\intercal \bfs_m(\bfx; \bftheta) \right)^2 \big],\label{eqn:ssm}
\end{multline}
the equivalence is summarized in the following theorem.
\begin{theorem}\label{app:thm:1}
Under some regularity conditions (\assref{ass:score}-\ref{ass:boundary} in \appref{app:basic}), we have
\begin{align}
    L(\bftheta; p_\bfv) = J(\bftheta; p_\bfv) + \up{C}, \label{eqn:thm1}
\end{align}
where $\up{C}$ is a constant \wrt $\bftheta$.
\end{theorem}
Other than our requirements on $p_\bfv$, the assumptions are exactly the same as in Theorem 1 of \citet{hyvarinen2005estimation}. We advise the interested readers to read \appref{app:basic} for technical statements of the assumptions and a rigorous proof of the theorem.

In practice, given a dataset $\bfx_1^N$, we draw $M$ projection vectors independently for each point $\bfx_i$ from $p_\bfv$. The collection of all such vectors $\{\bfv_{ij}\}_{1\leq i \leq N, 1\leq j \leq M}$ are abbreviated as $\bfv_{11}^{NM}$.%
We then use the following unbiased estimator of $J(\bftheta; p_\bfv)$
\begin{multline}
    \hat{J}(\bftheta; \bfx_1^N, \bfv_{11}^{NM}) \triangleq 
    \frac{1}{N} \frac{1}{M}  \sum_{i=1}^N \sum_{j=1}^M \bfv_{ij}^\intercal \nabla_\bfx \bfs_m(\bfx_i; \bftheta) \bfv_{ij} \\
    + \frac{1}{2} \left(\bfv_{ij}^\intercal \bfs_m(\bfx_i; \bftheta) \right)^2.
\end{multline}
Note that when $p_\bfv$ is a multivariate standard normal or multivariate Rademacher distribution, we have $\mbb{E}_{p_\bfv}[(\bfv^\intercal \bfs_m(\bfx; \bftheta))^2] = \norm{\bfs_m(\bfx; \bftheta)}_2^2$, in which case the second term of $J(\bftheta; p_\bfv)$ can be integrated analytically, and may lead to an estimator with reduced variance:
\begin{multline}
     \hat{J}_{\text{vr}}(\bftheta; \bfx_1^N, \bfv_{11}^{NM}) \triangleq 
    \frac{1}{N} \frac{1}{M}  \sum_{i=1}^N \sum_{j=1}^M \bfv_{ij}^\intercal \nabla_\bfx \bfs_m(\bfx_i; \bftheta) \bfv_{ij} \\
    + \frac{1}{2} \norm{\bfs_m(\bfx_i; \bftheta)}_2^2.
\end{multline}
Empirically, we do find that $\hat{J}_{\text{vr}}$ has better performance than $\hat{J}$. We refer to this version as sliced score matching with variance reduction (SSM-VR). In fact, we can leverage $\mbb{E}_{p_\bfv}[(\bfv^\intercal \bfs_m(\bfx; \bftheta))^2] = \norm{\bfs_m(\bfx; \bftheta)}_2^2$ to create a control variate for
guaranteed reduction of variance (Appendix~\ref{app:sec:vr}).  $\hat{J}_{\text{vr}}$ is also closely related to Hutchinson's trace estimator~\citep{hutchinson1990stochastic}, which we will analyze later in \secref{sec:connection}.

For sliced score matching, the second derivative term $\bfv^\intercal \nabla_\bfx \bfs_m(\bfx; \bftheta) \bfv$ is much less computationally expensive than $\operatorname{tr}(\nabla_\bfx \bfs_m(\bfx; \bftheta))$. Using auto-differentiation systems that support higher order gradients, we can compute it using two gradient operations for a single $\bfv$, as shown in \algoref{alg:ssm}. Similarly, when there are $M$ $\bfv$'s, the total number of gradient operations is $M + 1$. In contrast, assuming the dimension of $\bfx$ is $D$ and $D \gg M$, we typically need $D + 1$ gradient operations to compute $\operatorname{tr}(\nabla_\bfx \bfs_m(\bfx; \bftheta))$ because each diagonal entry of $\nabla_\bfx \bfs_m(\bfx; \bftheta)$ needs to be computed separately (see \algoref{alg:sm} in the appendix).

\algnewcommand{\IfThenElse}[3]{%
\algorithmicif\ #1\ \algorithmicthen\ #2\ \algorithmicelse\ #3}

\begin{algorithm}
	\caption{Sliced Score Matching}
	\label{alg:ssm}
	\begin{algorithmic}[1]
	    \Require{$\tilde{p}_m(\cdot; \bftheta), \bfx, \bfv$}
        \State{$\bfs_m(\bfx;\bftheta) \gets \texttt{grad}(\log \tilde{p}_m(\bfx; \bftheta), \bfx)$}
        \State{$\bfv^\intercal \nabla_\bfx \bfs_m(\bfx;\bftheta) \gets \texttt{grad}(\bfv^\intercal \bfs_m(\bfx;\bftheta), \bfx)$}
        \State{$J \gets \frac12(\bfv^\intercal \bfs_m(\bfx;\bftheta))^2$ (or $J \gets \frac12\norm{\bfs_m(\bfx;\bftheta)}_2^2$)}
        \State{$J \gets J + \bfv^\intercal \nabla_\bfx \bfs_m(\bfx;\bftheta) \bfv$}
        \item[]
        \Return{$J$}
	\end{algorithmic}
\end{algorithm}
In practice, we can tune $M$ to trade off variance and computational cost. In our experiments, we find that oftentimes {$\boxed{M=1}$} is already a good choice.

\subsection{SLICED SCORE ESTIMATION}\label{sec:sse}
As mentioned in section \ref{sec:bg-score-est}, the task of score estimation is to estimate $\nabla_\bfx \log q(\bfx)$ at some test point $\bfx$, given a set of samples $\bfx_1^N \overset{\text{\iid}}{\sim} q(\bfx)$. In what follows, we show how our sliced score matching objective $\hat{J}(\bftheta; \bfx_1^N, \bfv_{11}^{NM})$ can be straightforwardly adapted for this task.

We propose to use a vector-valued deep neural network $\bfh(\bfx; \bftheta): \mbb{R}^D \to \mbb{R}^D$ as our score model. Then, substituting $\bfh$ into $\hat{J}(\bftheta; \bfx_1^N, \bfv_{11}^{NM})$ for $\bfs_m$, we get
\begin{align*}
\resizebox{\columnwidth}{!}{$\displaystyle
    \frac{1}{N} \frac{1}{M}  \sum_{i=1}^N \sum_{j=1}^M \bfv_{ij}^\intercal \nabla_\bfx \bfh(\bfx_i; \bftheta) \bfv_{ij} 
    + \frac{1}{2} \left(\bfv_{ij}^\intercal \bfh(\bfx_i; \bftheta) \right)^2$},
\end{align*}
and we can optimize the objective to get $\hat{\bftheta}$. Afterwards, $\bfh(\bfx; \hat{\bftheta})$ can be used as an approximation to %
$\nabla_{\bfx}\log q(\bfx)$.\footnote{Note that $\bfh(\bfx; \bftheta)$ may not correspond to the gradient of any scalar function. For $\bfh(\bfx; \bftheta)$ to represent a gradient, one necessary condition is $\nabla_\bfx \times \bfh(\bfx; \bftheta) = \mbf{0}$ for all $\bfx$, which may not be satisfied by general networks. However, this is oblivious to the fact that $\bfh(\bfx; \hat{\bftheta})$ will be close to $\nabla_\bfx \log q_\bftheta(\bfx)$ in $\ell_2$ norm. As will be shown later, our argument based on integration by parts does not require $h(\bfx; \bftheta)$ to be a gradient.%
}

Using a similar argument of integration by parts (\cf, Eq.~\eqref{eqn:s_fisher}, \eqref{eqn:ssm} and \eqref{eqn:thm1}), we have
\begin{multline*}
    \mbb{E}_{p_\bfv}\mbb{E}_{p_d} \bigg[ \bfv^\intercal \nabla_\bfx \bfh(\bfx; \bftheta) \bfv + \frac{1}{2} (\bfv^\intercal \bfh(\bfx; \bftheta))^2  \bigg]=\\
    \frac{1}{2}\mbb{E}_{p_\bfv} \mbb{E}_{p_d}[(\bfv^\intercal \bfh(\bfx; \bftheta) - \bfv^\intercal \nabla_{\bfx}\log q(\bfx))^2] + \up{C},
\end{multline*}
which implies that minimizing $\hat{J}(\bftheta; \bfx_1^N,\bfv_{11}^{NM})$ with $\bfs_m(\bfx; \bftheta)$ replaced by $\bfh(\bfx; \bftheta)$ is approximately minimizing the average projected difference between $\bfh(\bfx; \bftheta)$ and $\nabla_{\bfx}\log q(\bfx)$. Hence, $\bfh(\bfx; \hat{\bftheta})$ should be close to $\nabla_{\bfx}\log q(\bfx)$ and can serve as a score estimator.

\section{THEORETICAL ANALYSIS}\label{sec:analysis}
In this section, we present several theoretical results to justify sliced score matching as a principled objective. We also discuss the connection of sliced score matching to other methods. For readability, we will defer rigorous statements of assumptions and theorems to the Appendix.
\subsection{CONSISTENCY}\label{sec:main_consistency}
One important question to ask for any statistical estimation objective is whether the estimated parameter is consistent under reasonable assumptions. Our results confirm that for any $M$, the objective $\hat{J}(\bftheta; \bfx_1^N, \bfv_{11}^{NM})$ 
is consistent under suitable assumptions 
as $N \to \infty$.

We need several standard assumptions to prove the results rigorously. Let $p_m$ be the normalized density induced by our unnormalized model $\tilde{p}_m$, which is assumed to be normalizable. First, we assume $\Theta$ is compact (\assref{ass:compact}), and our model family $\{ p_m(\bfx; \bftheta): \bftheta \in \Theta \}$ is well-specified and identifiable (\assref{ass:identifiability}). These are standard assumptions used for proving the consistency of MLE~\citep{vaart_1998}. We also adopt the assumption in \citet{hyvarinen2005estimation} that all densities are strictly positive (\assref{ass:positiveness}). Finally, we assume that $p_m(\bfx;\bftheta)$ has some Lipschitz properties (\assref{ass:lipschitz}), and $p_\bfv$ has bounded higher-order moments (\assref{ass:v}, true for all $p_\bfv$ considered in the experiments). Then, we can prove the consistency of $\hat{\bftheta}_{N,M} \triangleq \argmin_{\bftheta\in \Theta} \hat{J}(\bftheta; \bfx_1^N, \bfv_{11}^{NM})$:
\begin{theorem}[Consistency]\label{app:thm:consistency}
Assume the conditions of \thmref{app:thm:1} are satisfied. Assume further that the assumptions discussed above hold. Let $\bftheta^*$ be the true parameter of the data distribution. Then for every $M \in \mbb{N}^+$,
\begin{align*}
    \hat{\bftheta}_{N, M} \overset{p}{\to} \bftheta^*, \quad N \to \infty.
\end{align*}
\end{theorem}
\begin{proof}[Sketch of proof] 
We first prove that $J(\bftheta; p_\bfv) = 0 \Leftrightarrow \bftheta = \bftheta^*$ (\lemref{app:thm:2} and \thmref{app:thm:1}). Then we prove the uniform convergence of $\hat{J}(\bftheta; \bfx_1^N, \bfv_{11}^{NM})$ (\lemref{app:lemma:1}), which holds regardless of $M$. These two facts lead to consistency. For a complete proof, see Appendix~\ref{sec:consistency}.
\end{proof}
\begin{remark}
In \citet{hyvarinen2005estimation}, the authors only showed that $J(\bftheta) = 0 \Leftrightarrow \bftheta = \bftheta^*$, which leads to ``local consistency'' of score matching. This is a weaker notion of consistency compared to our settings.
\end{remark}
\subsection{ASYMPTOTIC NORMALITY}\label{sec:normality}
Asymptotic normality results can be very useful for approximate hypothesis testing and comparing different estimators. Below we show that $\hat{\bftheta}_{N,M}$ is asymptotically normal when $N \to \infty$.

In addition to the assumptions in \secref{sec:main_consistency}, we need an extra assumption to prove asymptotic normality. We require $p_m(\bfx; \bftheta)$ to have a stronger Lipschitz property (\assref{ass:lipschitz2}).

For simplicity, we denote $\nabla_\bfx h(\bfx)|_{\bfx = \bfx'}$ as $\nabla_\bfx h(\bfx')$, where $h(\cdot)$ is an arbitrary function. In the following, we will only show the asymptotic normality result for a specific $p_\bfv$. More general results are in Appendix~\ref{app:sec:normality}.
\begin{theorem}[Asymptotic normality, special case]\label{app:thm:normality}
Assume the assumptions discussed above hold. If $p_\bfv$ is the multivariate Rademacher distribution, we have
\begin{align*}
\sqrt{N}(\hat{\bftheta}_{N, M} - \bftheta^*) \overset{d}{\to} \mcal{N}(0, \Sigma),
\end{align*}
where
\begin{multline}
    \Sigma \triangleq
    [\nabla_\bftheta^2 J(\bftheta^*)]^{-1}\bigg(\sum_{1\leq i,j\leq D}V_{ij} + \frac{2}{M}\sum_{1\leq i\neq j \leq D}W_{ij}\bigg)\\
    \cdot [\nabla_\bftheta^2 J(\bftheta^*)]^{-1}.\label{eqn:variance}
\end{multline}
Here $D$ is the dimension of data; $V_{ij}$ and $W_{ij}$ are p.s.d matrices depending on $p_m(\bfx; \bftheta^*)$, and their definitions can be found in Appendix~\ref{app:sec:normality}.
\end{theorem}
\begin{proof}[Sketch of proof]
Once we get the consistency (\thmref{app:thm:consistency}), the rest closely follows the proof of asymptotic normality of MLE~\citep{vaart_1998}. A rigorous proof can be found in Appendix~\ref{app:sec:normality}.
\end{proof}
\begin{remark}
As expected, larger $M$ will lead to smaller asymptotic variance, as can be seen in Eq.~\eqref{eqn:variance}.
\end{remark}
\begin{remark}
As far as we know, there is no published proof of asymptotic normality for the standard (not sliced) score matching. However, by using the same techniques in our proofs, and under similar assumptions, we can conclude that the asymptotic variance of the score matching estimator is $[\nabla_\bftheta^2 J(\bftheta^*)]^{-1}\left(\sum_{ij}V_{ij}\right)[\nabla_\bftheta^2 J(\bftheta^*)]^{-1}$ (\corref{app:cor:sm}), which will always be smaller than sliced score matching with multivariate Rademacher projections. 
However, the gap reduces when $M$ increases.
\end{remark}

\subsection{CONNECTION TO OTHER METHODS}\label{sec:connection}
Sliced score matching is widely connected to many other methods, and can be motivated from some different perspectives. Here we discuss a few of them.
\paragraph{Connection to NCE.} Noise contrastive estimation (NCE), proposed by \citet{gutmann2010noise}, is another principle for training unnormalized statistical models. The method works by comparing $p_m(\bfx; \bftheta)$ with a noise distribution $p_n(\bfx)$.
We consider a special form of NCE which minimizes the following objective
\begin{align}
    -\mbb{E}_{p_d}[\log h(\bfx;\bftheta)] - \mbb{E}_{p_n}[\log (1 - h(\bfx; \bftheta))],\label{eqn:nce}
\end{align}
where $h(\bfx; \bftheta) \triangleq \frac{p_m(\bfx; \bftheta) }{p_m(\bfx; \bftheta) + p_m(\bfx - \bfv; \bftheta)}$, 
and we choose $p_n(\bfx) = p_d(\bfx + \bfv)$. Assuming $\|\bfv\|_2$ to be small, Eq.~\eqref{eqn:nce} can be written as the following by Taylor expansion
\begin{multline}
\frac{1}{4}\mbb{E}_{p_d}\bigg[\bfv^\intercal \nabla_\bfx \bfs_m (\bfx; \bftheta) \bfv + \frac{1}{2}(\bfs_m (\bfx; \bftheta)^\intercal \bfv)^2\bigg] \\+ 2\log 2 + o(\|\bfv\|_2^2). \label{eqn:nce2}
\end{multline}
For derivation, see \propref{prop:1} in the appendix. A similar derivation can also be found in \citet{gutmann2011bregman}. As a result of \eqref{eqn:nce2}, if we choose some $p_\bfv$ and take the expectation of \eqref{eqn:nce} \wrt $p_\bfv$, the objective will be equivalent to sliced score matching whenever $\norm{\bfv}_2 \approx 0$.

\paragraph{Connection to Hutchinson's Trick.} 
Hutchinson's trick~\citep{hutchinson1990stochastic} is a stochastic algorithm to approximate $\operatorname{tr}(A)$ for any square matrix $A$. The idea is to choose a distribution of a random vector $\bfv$ such that $\mbb{E}_{p_\bfv}[\bfv\bfv^\intercal] = I$, and then we have $\operatorname{tr}(A) = \mbb{E}_{p_\bfv}[\bfv^\intercal A \bfv]$. Hence, using Hutchinson's trick, we can replace $\operatorname{tr}(\nabla_\bfx \bfs_m(\bfx; \bftheta))$ with $\mbb{E}_{p_\bfv}[\bfv^\intercal \nabla_\bfx \bfs_m(\bfx; \bftheta) \bfv]$ in the score matching objective $J(\bftheta)$. Then the finite sample objective of score matching becomes
\begin{align*}
    \frac{1}{N} \sum_{i=1}^N  \bigg(\frac{1}{M}  \sum_{j=1}^M \bfv_{ij}^\intercal \nabla_\bfx \bfs_m(\bfx_i; \bftheta) \bfv_{ij}\bigg) 
    + \frac{1}{2}  \norm{\bfs_m(\bfx_i; \bftheta)}^2_2,
\end{align*}
which is exactly the sliced score matching objective with variance reduction $\hat{J}_{\text{vr}}(\bftheta; \bfx_1^N, \bfv_{11}^{NM})$.

\section{RELATED WORK} 
\subsection{SCALABLE SCORE MATCHING} \label{sec:sm_rel}
To the best of our knowledge, there are three existing methods that are able to scale up score matching to learning deep models on high dimensional data.

\paragraph{Denoising Score Matching.} \citet{vincent2011connection} proposes denoising score matching, a variant of score matching that completely circumvents the Hessian. Specifically, consider a noise distribution $q_\sigma(\tilde{\bfx} \mid \bfx)$, and let $q_\sigma(\tilde{\bfx}) = \int q_\sigma(\tilde{\bfx}\mid \bfx) p_d(\bfx)\ud \bfx$. Denoising score matching applies the original score matching to the noise-corrupted data distribution $q_\sigma(\tilde{\bfx})$, and the objective can be proven to be equivalent to the following up to a constant
\begin{multline*}
    \frac{1}{2}\mbb{E}_{q_\sigma(\tilde{\bfx}\mid \bfx)p_d(\bfx)}[\| \bfs_m(\tilde{\bfx};\bftheta) - \nabla_\bfx \log q_\sigma(\tilde{\bfx}\mid \bfx) \|_2^2],
\end{multline*}
which can be estimated without computing any Hessian. Although denoising score matching is much faster than score matching, it has obvious drawbacks. First, it can only recover the noise corrupted data distribution. Furthermore, choosing the parameters of the noise distribution is highly non-trivial and in practice the performance can be very sensitive to $\sigma$, and heuristics have to be used in practice. For example, \citet{saremi2018deep} propose a heuristic for choosing $\sigma$ based on Parzen windows.

\paragraph{Approximate Backpropagation.} 
\citet{kingma2010regularized} propose a backpropagation method to approximately compute the trace of the Hessian. Because the backpropagation of the full Hessian scales quadratically \wrt the layer size, the authors limit backpropagation only to the diagonal so that it has the same cost as the usual backpropagation for computing loss gradients. 
However, there are no theoretical guarantees for the approximation errors. In fact, the authors only did experiments on networks with a single hidden layer, in which case the approximation is exact. Moreover, there is no direct support for the proposed approximate backpropagation method in modern automatic differentiation frameworks.

\paragraph{Curvature Propagation.}
\citet{martens2012estimating} estimate the trace term in score matching by applying curvature propagation to compute an unbiased, complex-valued estimator of the diagonal of the Hessian. The work claims that curvature propagation will have the smallest variance among a class of estimators, which includes the Hutchinson's estimator. However, their proof 
evaluates the pseudo-variance of the complex-valued estimator instead of the variance. In practice, curvature propagation can have large variance when the number of nodes in the network is large, because it introduces noise for each node in the network. Finally, implementing curvature propagation also requires manually modifying the backpropagation code, handling complex numbers in neural networks, and will be inefficient for networks of more general structures, such as recurrent neural networks.

\subsection{KERNEL SCORE ESTIMATORS} \label{sec:score_est}
Two prior methods for score estimation
are based on a generalized form of Stein's identity~\citep{stein1981estimation,gorham2017measuring}:
\begin{align}
    \mbb{E}_{q(\bfx)}[\mbf{h}(\bfx) \nabla_\bfx \log q(\bfx)^\intercal + \nabla_\bfx \mbf{h}(\bfx)] = \mbf{0},\label{eqn:stein}
\end{align}
where $q(\bfx)$ is a continuously differentiable density and $\mbf{h}(\bfx)$ is a function satisfying some regularity conditions. 
\citet{li2018gradient} propose to 
set $\mbf{h}(\bfx)$ as the feature map of some kernel in the \emph{Stein class}~\citep{liu2016kernelized} of $q$,
and solve a finite-sample version of \eqref{eqn:stein} to obtain estimates of $\nabla_\bfx \log q(\bfx)$ at the sample points.
We refer to this method as \emph{Stein} in the experiments.
\citet{shi2018spectral} adopt a different approach but also exploit \eqref{eqn:stein}. They build their estimator by expanding $\nabla_\bfx \log q(\bfx)$ as a spectral series of eigenfunctions and solve for the coefficients using \eqref{eqn:stein}. %
Compared to \emph{Stein}, their method is argued to have theoretical guarantees and principled out-of-sample estimation at an unseen test point. We refer to their method as \emph{Spectral} in the experiments.
\section{EXPERIMENTS}
Our experiments include two parts: (1) to test the effectiveness of sliced score matching (\SSM{}) in learning deep models for density estimation, and (2) to test the ability of \SSM{} in providing score estimates for applications such as VAEs with implicit encoders and WAEs. 
Unless specified explicitly, we choose $M=1$ by default. 

\subsection{DENSITY ESTIMATION}

\begin{figure*}[h]
    \begin{center}
    \includegraphics[width=0.9\textwidth]{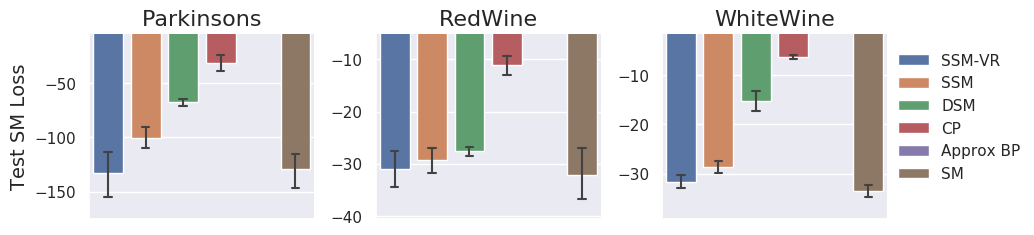}
     \vspace{-5mm}
    \end{center}
    \caption{\SM{} loss after training DKEF models on UCI datasets with different loss functions; lower is better. Results for approximate backprapogation are not shown because losses were larger than $10^{9}$.}
    \label{fig:dkef2}
\end{figure*}
\begin{figure}
    \centering
    \includegraphics[width=0.8\linewidth]{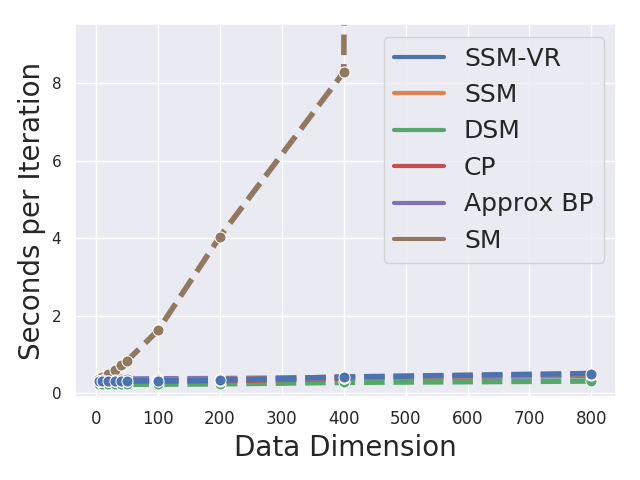}
    \vspace{-5mm}
    \caption{\SM{} performance degrades linearly with the data dimension, while efficient approaches have relatively similar performance.}
    \label{fig:scale}
\end{figure}

We evaluate \SSM{} and its variance-reduced version (\SSMVR{}) for density estimation and compare against score matching (\SM{}) and its three scalable baselines: denoising score matching (\DSM{}), approximate backpropagation (\aBP{}), and curvature propagation (\CP{}). All \SSM{} methods in this section use the multivariate Rademacher distribution as $p_\bfv$. Our results demonstrate that: (1) \SSM{} is comparable in performance to \SM{}, (2) \SSM{} outperforms other computationally efficient approximations to \SM{}, and (3) unlike \SM{}, \SSM{} scales effectively to high dimensional data.

\subsubsection{Deep Kernel Exponential Families}
\paragraph{Model.}
Deep kernel exponential families (DKEF) are unnormalized density models trained using \SM{} \citep{wenliang2018}. DKEFs parameterize the unnormalized log density as $\log\ptilde_f(\bfx) = f(\bfx) + \log q_0(\bfx)$, %
where $f$ is a mixture of Gaussian kernels evaluated at different inducing points: $f(\bfx) = \sum_{l=1}^{L}\alpha_l k(\bfx, \rvz_l)$. The kernel is defined on the feature space of a neural network, and the network parameters are trained along with the inducing points $\rvz_l$. Further details of the model, which is identical to that in \citet{wenliang2018}, are presented in Appendix \ref{app:dkef}.  

\paragraph{Setup.}  Following the settings of \citet{wenliang2018}, we trained models on three UCI datasets: Parkinsons, RedWine, and WhiteWine \citep{dua2019uci}, and used their original code %
for \SM{}. To compute the trace term exactly, %
\citet{wenliang2018}'s manual implementation of backpropagation takes over one thousand lines for a model that is four layers deep, while the implementation of \SSM{} only takes several lines. For \DSM{}, the value of $\sigma$ %
is chosen by grid search using the \SM{} loss on a validation dataset. %
All models are trained with 15 different random seeds and training is stopped when validation loss does not improve for 200 steps. 

\paragraph{Results.}
Results in \figref{fig:dkef2} demonstrate that \SSMVR{} performs comparably to \SM{}, when evaluated using the \SM{} loss on a held-out test set. 
We observe that variance reduction substantially improves the performance of \SSM{}. In addition, \SSM{} outperforms other computationally efficient approaches. \DSM{} can perform comparably to \SSM{} on RedWine. However, it is challenging to select $\sigma$ for \DSM{}. Models trained using $\sigma$ from the heuristic in \citet{saremi2018deep} are far from optimal (on both \SM{} losses and likelihoods), and extensive grid search is needed to find the best $\sigma$. \CP{} performs substantially worse, which is likely because it injects noise for each node in the computation graph, and the amount of noise introduced is too large for a neural-network-based kernel evaluated at 200 inducing points, which supports our hypothesis that \CP{} does not work effectively for deep models. Results for \aBP{} are omitted because the losses exceeded $10^{9}$ on all datasets. This is because \aBP{} provides a biased estimate of the Hessian without any error guarantees. All the results are similar when evaluating according to log-likelihood metric (Appendix \ref{app:dkef}). 

\paragraph{Scalability to High Dimensional Data.}
We evaluate the computational efficiency of different losses on data of increasing dimensionality. We fit DKEF models to a multivariate standard normal in high dimensional spaces.  
The average running time per minibatch of 100 examples is reported in \figref{fig:scale}. \SM{} performance degrades linearly with the dimension of the input data due to the computation of the Hessian, and %
creates out of memory errors for a 12GB GPU after the dimension increases beyond 400. \SSM{} performs similarly to \DSM{}, \aBP{} and \CP{}, and they are all scalable to large dimensions.%

\begin{table}
\centering
\begin{tabular}{lrr}
\toprule
{} &  Test \SM{} Loss &  Test LL \\
\midrule
\MLE{}               &          -579 &     \textbf{-791} \\
\SSMVR{}        &         \textbf{-8054} &    -3355 \\
\SSM{}            &         -2428 &    -2039 \\
\DSM{} ($\sigma=0.10$) &         -3035 &    -4363 \\
\DSM{} ($\sigma=1.74$) &           -97 &    -8082 \\
\CP{}                &         -1694 &    -1517 \\
\ABP{}         &           -48 &    -2288 \\
\bottomrule
\end{tabular}
\caption{Score matching losses and log-likelihoods for NICE models on MNIST. $\sigma=0.1$ is by grid search and $\sigma=1.74$ is from the heuristic of \citet{saremi2018deep}. 
}
\label{tab:nice}
\end{table}

\subsubsection{Deep Flow Models}
\paragraph{Setup.} As a sanity check, we also evaluate %
\SSM{} by training a NICE flow model~\citep{dinh14nice}, whose likelihood is tractable and can be compared to results obtained by \MLE{}. The model we use has
20 hidden layers, and 1000 units per layer. Models are trained to fit MNIST handwritten digits, which are 784 dimensional images. Data are dequantized by adding uniform noise in the range $[-\frac{1}{512}, \frac{1}{512}]$, and transformed using a logit transformation, $\log(x) - \log(1 - x)$. We provide additional details in Appendix~\ref{app:nice}.

Training with \SM{} is extremely computationally expensive in this case. Our \SM{} implementation based on auto-differentiation takes around 7 hours to finish one epoch of training, and 12 GB GPU memory cannot hold a batch size larger than 24, so we do not include these results. Since NICE has tractable likelihoods, we also evaluate \MLE{} as a surrogate objective for minimizing the \SM{} loss. Notably, likelihoods and \SM{} losses might be uncorrelated when the model is mis-specified, which is likely to be the case for a complex dataset like MNIST.

\paragraph{\color{black}Results.} 
\SM{} losses and log-likelihoods on the test dataset are reported in \tabref{tab:nice}, where models are evaluated using the best checkpoint in terms of the \SM{} loss on a validation dataset over 100 epochs of training. \SSMVR{} greatly outperforms all the other methods on the \SM{} loss. \DSM{} performs worse than \SSMVR{}, and $\sigma$ is still hard to tune. Specifically, following the heuristic in \cite{saremi2018deep} leads to $\sigma=1.74$, which performed the worst (on both log-likelihood and \SM{} loss) of the eight choices of $\sigma$ in our grid search. \ABP{} has more success on NICE than for training DKEF models. This might be because the Hessians of hidden layers of NICE are closer to a diagonal matrix, which results in a smaller approximation error for \aBP{}. As in the DKEF experiments, \CP{} performs worse. This is likely due to injecting noise to all hidden units, which will lead to large variance for a network as big as NICE. Unlike the DKEF experiments, we find that good log-likelihoods are less correlated with good \SM{} loss, probably due to model mis-specification.

\subsection{SCORE ESTIMATION}
We consider two typical tasks that require accurate score estimations: (1) training VAEs with an implicit encoder and (2) training Wasserstein Auto-Encoders. We show in both tasks \SSM{} outperforms previous score estimators. Samples generated by various algorithms can be found in Appendix~\ref{app:sec:samples}.
\begin{table}
    \centering
    \begin{adjustbox}{max width=0.8\linewidth}
    \begin{tabular}{c|cc||cc}
    	\Xhline{3\arrayrulewidth} \bigstrut
    	& \multicolumn{2}{c||}{VAE} & \multicolumn{2}{c}{WAE}\\
    	\Xhline{1\arrayrulewidth}\bigstrut
    	Latent Dim &  8 & 32 & 8 & 32\\
    	\Xhline{1\arrayrulewidth}\bigstrut
        \ELBO{} & 96.87 & \textbf{89.06} & N/A  & N/A \\
        \SSM{} & \textbf{95.50} & 89.25 (\textbf{88.29}$^\dagger$) & \textbf{98.24} & \textbf{90.37}\\
        \Stein{} & 96.71 & 91.84 & 99.05 & 91.70 \\
        \Spectral{} & 96.60 & 94.67 & 98.81 & 92.55\\
        \Xhline{3\arrayrulewidth}
    \end{tabular}
    \end{adjustbox}
    \caption{Negative log-likelihoods on MNIST, estimated by AIS. $^\dagger$The result of \SSM{} with $M=100$, in which case \SSM{} matches the computational cost of kernel methods, which used 100 samples for each data point.}
    \label{tab:mnist}
\end{table}

\begin{table}
	\centering
	\begin{adjustbox}{max width=\linewidth}
		\begin{tabular}{c|c|cccccc}
			\Xhline{3\arrayrulewidth} \bigstrut
             & \diagbox{Method}{Iteration} & 10k & 40k & 70k & 100k\\
			\Xhline{1\arrayrulewidth}\bigstrut
			\multirow{4}{*}[-0.9em]{VAE} & \ELBO{} & \textbf{96.20} & 73.70 & 69.42 & 66.32 \\\bigstrut
		    & \SSM{} & 108.52 & \textbf{70.28} & \textbf{66.52} & \textbf{62.50}\\ \bigstrut
		    & \Stein{} & 126.60 & 118.87 & 120.51 & 126.76\\\bigstrut
		    & \Spectral{} & 131.90 & 125.04 & 128.36 & 133.93\\\Xhline{1\arrayrulewidth}\bigstrut
		    \multirow{3}{*}[-0.5em]{WAE} & \SSM{} & 84.11 & \textbf{61.09} & \textbf{56.23} & \textbf{54.33}\\\bigstrut
		    & \Stein{} & 82.93 & 63.46 & 58.53 & 57.61\\\bigstrut
		    & \Spectral{} & \textbf{82.30} & 62.47 & 58.03 & 55.96\\
			\Xhline{3\arrayrulewidth}
		\end{tabular}
	\end{adjustbox}
	\caption{FID scores of different methods versus number of training iterations on CelebA dataset.}\label{tab:celeba}
\end{table}

\subsubsection{VAE with Implicit Encoders}
\paragraph{Background.} Consider a latent variable model $p(\bfx, \bfz)$, where $\bfx$ and $\bfz$ denote observed and latent variables respectively. A variational auto-encoder (VAE) is composed of two parts: 1) an encoder $p_\bftheta(\bfx \mid \bfz)$ modeling the conditional distribution of $\bfx$ given $\bfz$; and a decoder $q_\bfphi(\bfz \mid \bfx)$ that approximates the posterior distribution of the latent variable. A VAE is typically trained by maximizing the following evidence lower bound (\ELBO{})
\begin{align*}
    \mbb{E}_{p_d}[\mbb{E}_{q_\bfphi(\bfz\mid \bfx)} \log p_\bftheta(\bfx \mid \bfz)p(\bfz) - \mbb{E}_{q_\phi(\bfz \mid \bfx)} \log q_\bfphi(\bfz \mid \bfx)],
\end{align*}
where $p(\bfz)$ denotes a pre-specified prior distribution. The expressive power of $q_\bfphi(\bfz \mid \bfx)$ is critical to the success of variational learning. Typically, $q_\bfphi(\bfz \mid \bfx)$ is chosen to be a simple distribution with tractable densities so that $H(q_\bfphi) \triangleq - \mbb{E}_{q_\phi(\bfz \mid \bfx)} \log q_\bfphi(\bfz \mid \bfx)$ is tractable. We call this traditional approach ``\ELBO{}'' in the experiments. With score estimation techniques, we can directly compute $\nabla_\bfphi H(q_\bfphi)$ for implicit distributions, which enables more flexible options for $q_\bfphi$. We consider 3 score estimation techniques: \SSM{}, \Stein{}~\citep{li2018gradient} and \Spectral{}~\citep{shi2018spectral}.

For a single data point $\bfx$, kernel score estimators need multiple samples from $q_\bfphi(\bfz \mid \bfx)$ to estimate its score. On MNIST, we use 100 samples, as done in \citet{shi2018spectral}. On CelebA, however, we can only take 20 samples due to GPU memory limitations. In contrast, \SSM{} learns a score network $h(\bfz \mid \bfx)$ along with $q_\bfphi(\bfz \mid \bfx)$, which amortizes the cost of score estimation. Unless noted explicitly, we use one projection per data point ($M=1$) for \SSM, which is scalable to deep networks.

\paragraph{Setup.} We consider VAE training on MNIST and CelebA~\citep{liu2015faceattributes}. All images in CelebA are first cropped to a patch of $140 \times 140$ and then resized to $64 \times 64$. We report negative likelihoods on MNIST, as estimated by AIS~\citep{neal2001annealed} with 1000 intermediate distributions. We evaluate sample quality on CelebA with FID scores~\citep{heusel2017gans}. For fast AIS evaluation on MNIST, we use shallow fully-connected networks with 3 hidden layers. For CelebA experiments we use deep convolutional networks. The architectures of implicit encoders and score networks are straightforward modifications to the encoders of \ELBO{}. More details are provided in \appref{app:sec:vi}.

\paragraph{Results.} The negative likelihoods of different methods on MNIST are reported in the left part of \tabref{tab:mnist}. We note that \SSM{} outperforms \Stein{} and \Spectral{} in all cases. When the latent dimension is 8, \SSM{} outperforms \ELBO{}, which indicates that the expressive power of implicit $q_\bfphi(\bfz \mid \bfx)$ pays off. When the latent dimension is 32, the gaps between \SSM{} and kernel methods are even larger, and the performance of \SSM{} is still comparable to \ELBO{}. Moreover, when $M=100$ (matching the computation of kernel methods), \SSM{} outperforms \ELBO{}.

For CelebA, we provide FID scores of samples in the top part of \tabref{tab:celeba}. We observe that after 40k training iterations, 
\SSM{} outperforms all other baselines. Kernel methods perform poorly in this case because only 20 samples per data point can be used due to limited amount of GPU memory. 
Early during training, \SSM{} does not perform as well. Since the score network is trained along with the encoder and decoder, a moderate number of training steps is needed to give an accurate score estimation (and better learning of the VAE).

\subsubsection{WAEs}
\paragraph{Background.} WAE is another method to learn latent variable models, which generally produces better samples than VAEs. Similar to a VAE, it contains an encoder $q_\bfphi(\bfz \mid \bfx)$ and a decoder $p_\bftheta(\bfx \mid \bfz)$ and both can be implicit distributions. Let $p(\bfz)$ be a pre-defined prior distribution, and $q_\bfphi(\bfz) \triangleq \int q_\bfphi(\bfz \mid \bfx) p_d(\bfx) \ud \bfx$ denote the aggregated posterior distribution. Using a metric function $c(\cdot, \cdot)$ and KL divergence between $q_\bfphi(\bfz)$ and $p(\bfz)$, WAE minimizes the following objective
\begin{align*}
    \mbb{E}_{p_d}[\mbb{E}_{q_\bfphi(\bfz \mid \bfx)}[c(\bfx, p_\bftheta(\bfx \mid \bfz)) - \lambda \log p(\bfz)]]
    - \lambda H(q_\bfphi(\bfz)).
\end{align*}
Compared to $\nabla_\bfphi H(q_\bfphi(\bfz \mid \bfx))$, it is easier to estimate $\nabla_\bfphi H(q_\bfphi(\bfz))$ for kernel methods, because the samples of $q_\bfphi(\bfz)$ can be collected by first sampling $\bfx_1, \bfx_2, \cdots, \bfx_N \stackrel{\text{\iid}}{\sim} p_d(\bfx)$ and then sample one $\bfz$ for each $\bfx_i$ from $q_\phi(\bfz \mid \bfx_i)$. In constrast, multiple $\bfz$'s need to be sampled for each $\bfx_i$ when estimating $\nabla_\bfphi H(q_\bfphi(\bfz \mid \bfx))$ with kernel approaches. For \SSM, we use a score network $h(\bfz)$ and train it alongside $q_\bfphi(\bfz \mid \bfx)$.

\paragraph{Setup.} The setup for WAE experiments is largely the same as VAE. The architectures are very similar to those used in the VAE experiments, and the only difference is that we made decoders implicit, as suggested in~\citet{tolstikhin2018wasserstein}. More details can be found in Appendix~\ref{app:sec:wae}.

\paragraph{Results.} The negative likelihoods on MNIST are provided in the right part of \tabref{tab:mnist}. \SSM{} outperforms both kernel methods, and achieves a larger performance gap when the latent dimension is higher. 
The likelihoods are lower than the VAE results as the WAE objective does not directly maximize likelihoods. 

We show FID scores for CelebA experiments in the bottom part of \tabref{tab:celeba}. As expected, kernel methods perform much better than before, because it is faster to sample from $q_\bfphi(\bfz)$. The FID scores are generally lower than those in VAE experiments, which supports previous results that WAEs generally obtain better samples. \SSM{} outperforms both kernel methods when the number of iterations is more than 40k. This shows the advantages of training a deep, expressive score network compared to using a fixed kernel in score estimation tasks.
\section{CONCLUSION}
We propose sliced score matching, a scalable method for learning unnormalized statistical models and estimating scores for implicit distributions. %
Compared to the original score matching and its variants, our estimator can scale to deep models on high dimensional data, while remaining easy to implement in modern automatic differentiation frameworks. Theoretically, our estimator is consistent and asymptotically normal under some regularity conditions. %
Experimental results demonstrate that our method outperforms competitors in learning deep energy-based models and provides more accurate estimates than kernel score estimators in training implicit VAEs and WAEs.

\subsection*{Acknowledgements}
This research was supported by Intel Corporation, TRI, NSF (\#1651565, \#1522054, \#1733686 ), ONR  (N00014-19-1-2145), AFOSR (FA9550-
19-1-0024).%

\bibliography{ssm}
\allowdisplaybreaks

\newpage
\appendix
\onecolumn
\section{Samples}\label{app:sec:samples}
\subsection{VAE WITH IMPLICIT ENCODERS}
\subsubsection{MNIST}
\begin{table}[H]
    \centering
    \begin{tabular}{c>{\centering\arraybackslash}m{4.4cm}>{\centering\arraybackslash}m{4.4cm}}
          & Latent Dim 8 & Latent Dim 32\\
    \begin{tabular}{l} ELBO \end{tabular} & \includegraphics[width=0.28\textwidth]{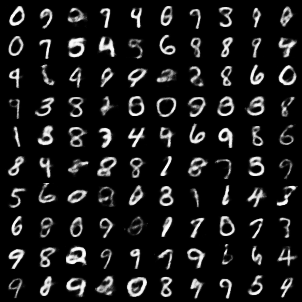} & \includegraphics[width=0.28\textwidth]{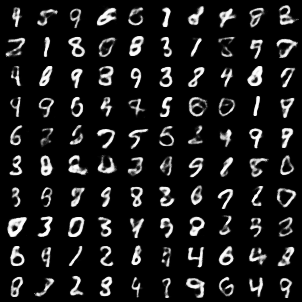}\\
        SSM &\includegraphics[width=0.28\textwidth]{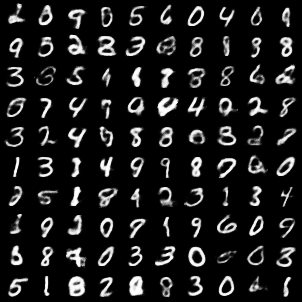} & \includegraphics[width=0.28\textwidth]{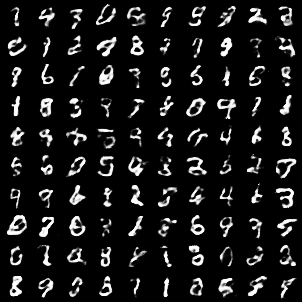}\\
        Stein & \includegraphics[width=0.28\textwidth]{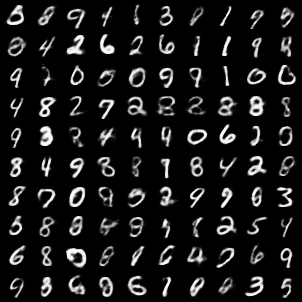} & \includegraphics[width=0.28\textwidth]{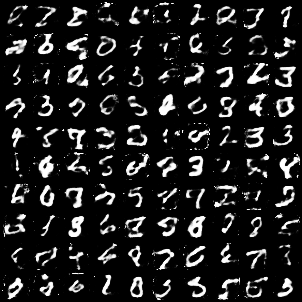}\\
        Spectral & \includegraphics[width=0.28\textwidth]{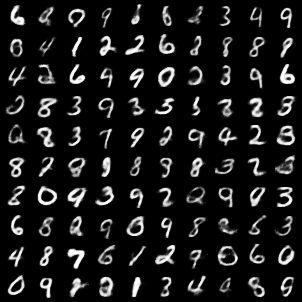} & \includegraphics[width=0.28\textwidth]{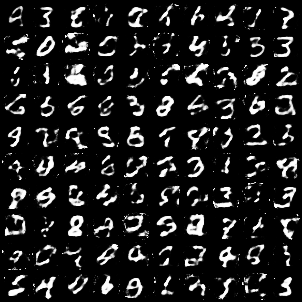}\\
    \end{tabular}
    \hspace{2cm}\caption{VAE samples on MNIST.}
\end{table}
\subsubsection{CelebA}
\FloatBarrier
\begin{center}
\begin{table}[H]
    \centering
    \begin{tabular}{cc}
    (a) ELBO & (b) SSM\\
    \includegraphics[width=0.4\textwidth]{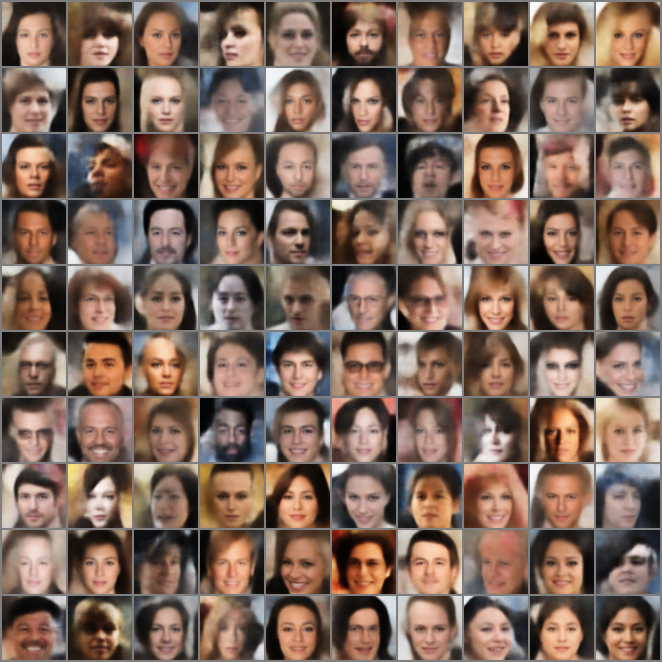}  &\includegraphics[width=0.4\textwidth]{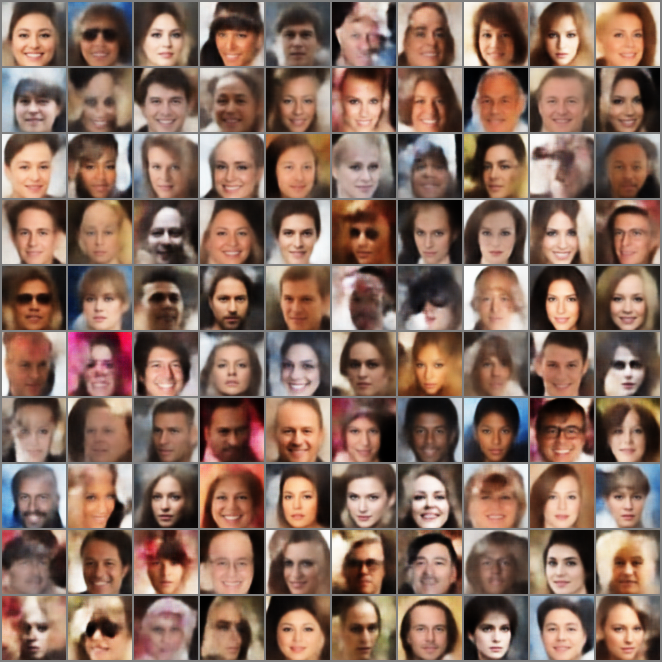}
        \vspace{5mm}\\
    (c) Stein  & (d) Spectral \\
    \includegraphics[width=0.4\textwidth]{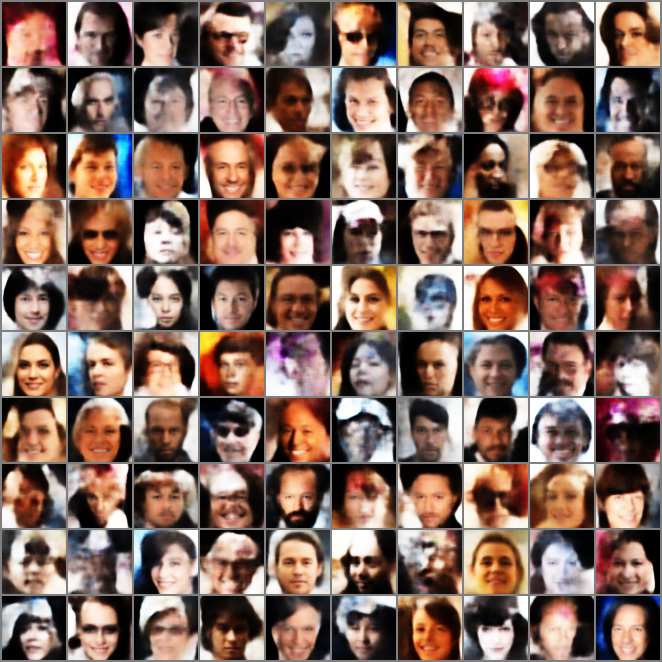} &  \includegraphics[width=0.4\textwidth]{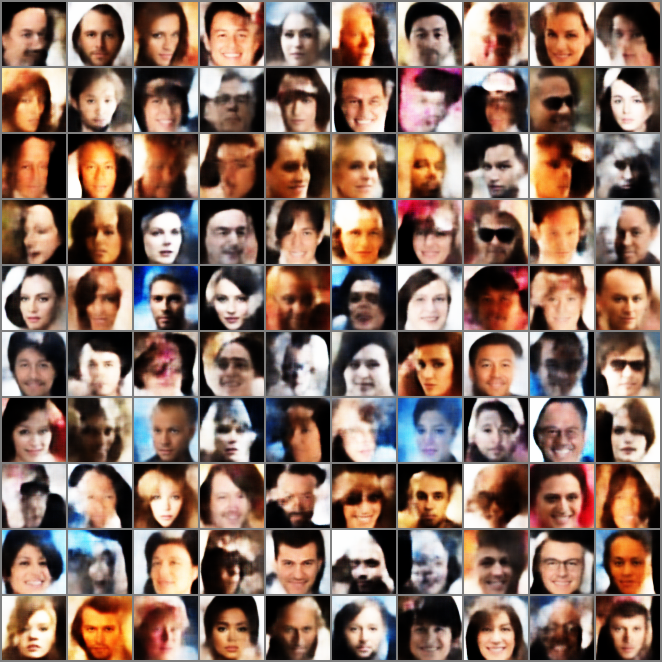}\\
    \end{tabular}
    \caption{VAE samples on CelebA.}
\end{table}
\end{center}
\FloatBarrier
\subsection{WAE}
\subsubsection{MNIST}
\begin{center}
\begin{table}[H]
    \centering
    \begin{tabular}{c>{\centering\arraybackslash}m{4.4cm}>{\centering\arraybackslash}m{4.4cm}}
         & Latent Dim 8 & Latent Dim 32\\
        SSM &\includegraphics[width=0.28\textwidth]{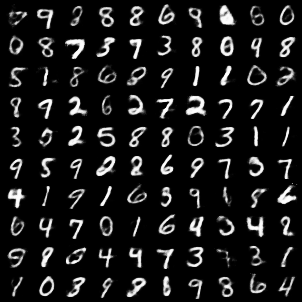} & \includegraphics[width=0.28\textwidth]{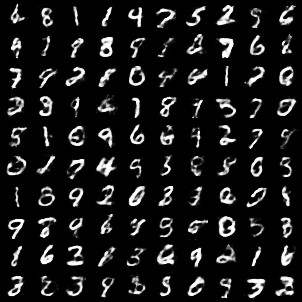}\\
        Stein & \includegraphics[width=0.28\textwidth]{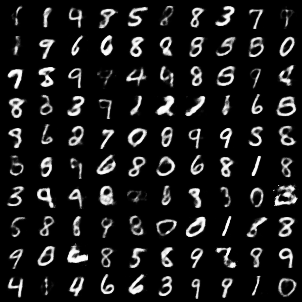} & \includegraphics[width=0.28\textwidth]{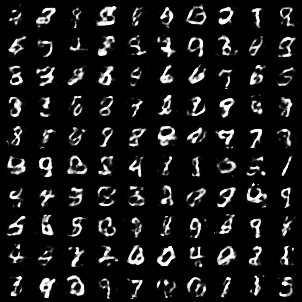}\\
        Spectral & \includegraphics[width=0.28\textwidth]{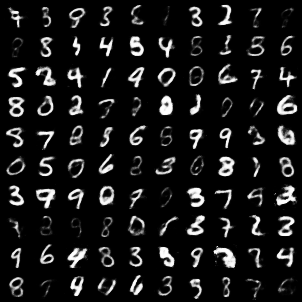} & \includegraphics[width=0.28\textwidth]{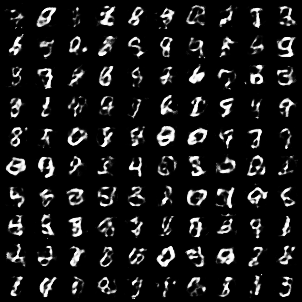}\\
    \end{tabular}
    \hspace{2cm}\caption{WAE samples on MNIST.}
\end{table}
\end{center}
\subsubsection{CelebA}
\begin{table}[H]
    \centering
    \begin{tabular}{c>{\centering\arraybackslash}m{6cm}}
        SSM &\includegraphics[width=0.4\textwidth]{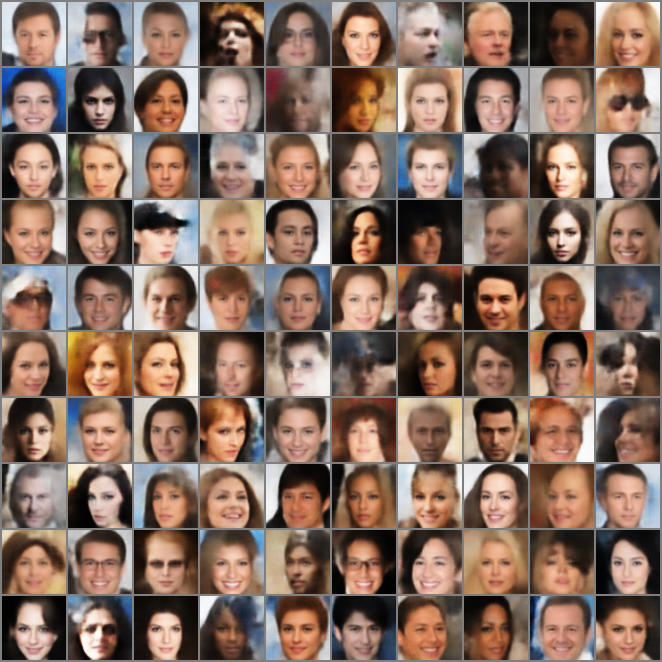}\\
        Stein & \includegraphics[width=0.4\textwidth]{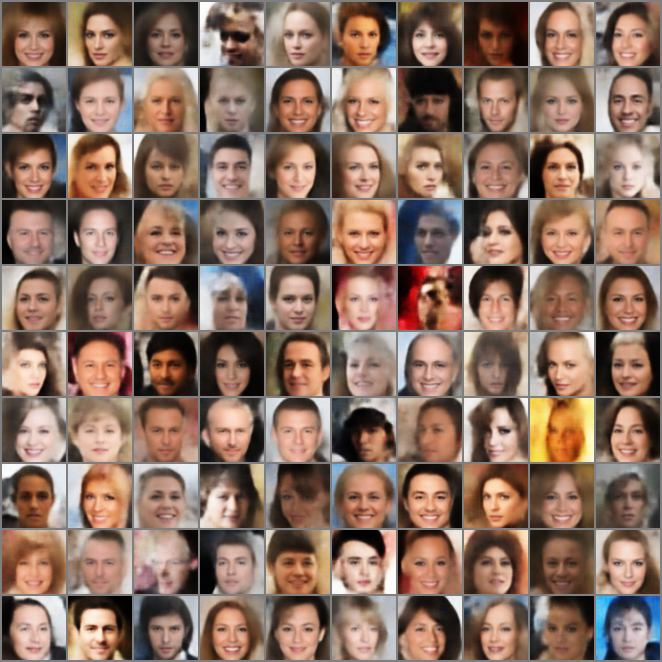}\\
        Spectral & \includegraphics[width=0.4\textwidth]{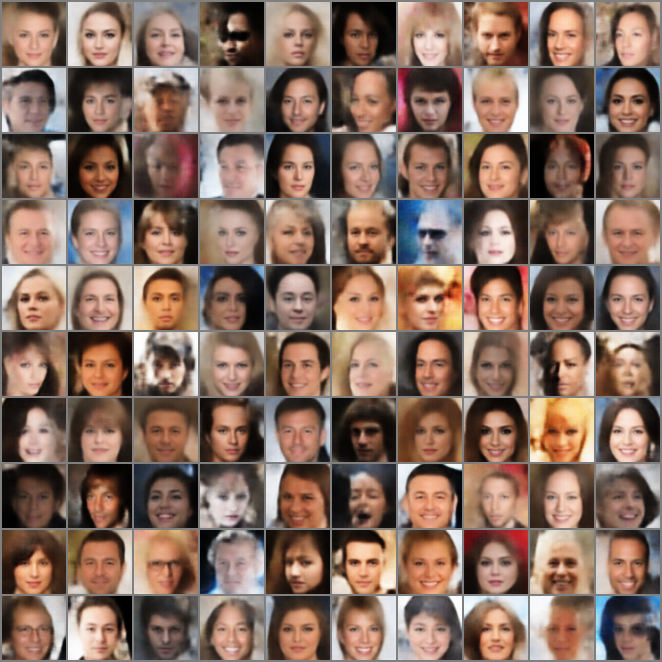}\\
    \end{tabular}
    \hspace{2cm}\caption{WAE samples on CelebA.}
\end{table}
\section{PROOFS}
\subsection{NOTATIONS}
Below we provide a summary of the most commonly used notations used in the proofs. First, we denote the data distribution as $p_d(\bfx)$ and assume that the training/test data $\{\bfx_1, \bfx_2, \cdots, \bfx_N\}$ are \iid samples of $p_d(\bfx)$. The model is denoted as $p_m(\bfx; \bftheta)$, where $\bftheta$ is restricted to a parameter space $\Theta$. Note that $p_m(\bfx;\bfv)$ can be \emph{an unnormalized energy-based model}. We use $\bfv$ to represent a random vector with the same dimension of input $\bfx$. This vector $\bfv$ is often called the \emph{projection vector}, and we use $p_\bfv$ to denote its distribution.

Next, we introduce several shorthand notations for quantities related to $p_m(\bfx; \bftheta)$ and $p_d(\bfx)$. The log-likelihood $\log p_m(\bfx; \bftheta)$ and $\log p_d(\bfx)$ are respectively denoted as $l_m(\bfx; \bftheta)$ and $l_d(\bfx)$. The (Stein) score function $\nabla_\bfx \log p_m(\bfx; \bftheta)$ and $\nabla_\bfx \log p_d(\bfx)$ are written as $\bfs_m(\bfx; \bftheta)$ and $\bfs_d(\bfx)$, and finally the Hessian of $\log p_m(\bfx; \bftheta)$ \wrt $\bfx$ is denoted as $\nabla_\bfx \bfs_m(\bfx; \bftheta)$.

We also adopt some convenient notations for collections. In particular, we use $\mbf{x}_1^N$ to denote a collection of $N$ vectors $\{\bfx_1, \bfx_2, \cdots, \bfx_N\}$ and use $\mbf{v}_{11}^{NM}$ to denote $N\times M$ vectors $\{\bfv_{11}, \bfv_{12}, \cdots, \bfv_{1M}, \bfv_{21}, \allowbreak \bfv_{22}, \cdots, \bfv_{2M}, \cdots, \bfv_{N1}, \bfv_{N2}, \cdots, \bfv_{NM}\}$. 

\subsection{BASIC PROPERTIES} \label{app:basic}
The following regularity conditions are needed for integration by parts and identifiability.
\begin{assumption}[Regularity of score functions]\label{ass:score}
The model score function $\bfs_m(\bfx)$ and data score function $\bfs_d(\bfx)$ are both differentiable. They additionally satisfy $\mbb{E}_{p_d}[\norm{\bfs_m(\bfx)}_2^2] < \infty$ and $\mbb{E}_{p_d}[\norm{\bfs_d(\bfx)}_2^2] < \infty$.
\end{assumption}
\begin{assumption}[Regularity of projection vectors]\label{ass:v}
The projection vectors satisfy $\mbb{E}_{p_\bfv}[\norm{\bfv}^2_2] < \infty$, and $\mbb{E}_{p_\bfv}[\bfv \bfv^\intercal] \succ 0$.
\end{assumption}
\begin{assumption}[Boundary conditions]\label{ass:boundary}
$\forall \bftheta \in \Theta, \lim_{\norm{\bfx}\rightarrow \infty} \bfs_m(\bfx; \bftheta) p_d(\bfx) = 0$.
\end{assumption}
\begin{assumption}[Identifiability]\label{ass:identifiability}
The model family $\{p_m(\bfx; \bftheta) \mid \bftheta \in \Theta \}$ is well-specified, \ie, $p_d(\bfx) = p_m(\bfx; \bftheta^*)$. Furthermore, $p_m(\bfx; \bftheta) \neq p_m(\bfx; \bftheta^*)$ whenever $\bftheta \neq \bftheta^*$.
\end{assumption}
\begin{assumption}[Positiveness]\label{ass:positiveness}
$p_m(\bfx; \bftheta) > 0, \forall \bftheta \in \Theta$ and $\forall \bfx$.
\end{assumption}

\begin{customthm}{\ref{app:thm:1}}
Assume $\bfs_m(\bfx;\bftheta)$, $\bfs_d(\bfx)$ and $p_\bfv$ satisfy some regularity conditions (\assref{ass:score}, \assref{ass:v}). Under proper boundary conditions (\assref{ass:boundary}), we have
\begin{align}
    L(\bftheta; p_\bfv) \triangleq \frac{1}{2}  \mbb{E}_{p_\bfv} \mbb{E}_{p_d}\left[ (\bfv^\intercal \bfs_m (\bfx; \bftheta) - \bfv^\intercal \bfs_d(\bfx))^2 \right]
    = \mbb{E}_{p_\bfv} \mbb{E}_{p_d}\bigg[\bfv^\intercal \nabla_\bfx \bfs_m(\bfx; \bftheta) \bfv + 
    \frac{1}{2} \left(\bfv^\intercal \bfs_m(\bfx; \bftheta) \right)^2 \bigg] + \up{C}, \label{app:t1}
\end{align}
where $\up{C}$ is a constant \wrt $\bftheta$.
\end{customthm} 
\begin{proof}
The basic idea of this proof is similar to that of Theorem 1 in \citet{hyvarinen2005estimation}. First, note that $L(\bftheta, p_\bfv)$ can be expanded to
\begin{align}
    L(\bftheta, p_\bfv) &= \frac{1}{2} \mbb{E}_{p_\bfv} \mbb{E}_{p_d}\left[ (\bfv^\intercal \bfs_m (\bfx; \bftheta) - \bfv^\intercal \bfs_d(\bfx))^2 \right] \notag\\
    &\stackrel{(i)}{=}  \frac{1}{2} \mbb{E}_{p_\bfv} \mbb{E}_{p_d} [(\bfv^\intercal \bfs_m(\bfx; \bftheta))^2 + (\bfv^\intercal \bfs_d(\bfx))^2 - 2(\bfv^\intercal \bfs_m(\bfx; \bftheta))(\bfv^\intercal \bfs_d(\bfx; \bftheta))] \label{app:t1_expand}\\
    &= \mbb{E}_{p_\bfv} \mbb{E}_{p_d}\bigg[-(\bfv^\intercal \bfs_m(\bfx; \bftheta))(\bfv^\intercal \bfs_d(\bfx; \bftheta)) + 
    \frac{1}{2} \left(\bfv^\intercal \bfs_m(\bfx; \bftheta) \right)^2 \bigg] + \up{C},
\end{align}
where $(i)$ is due to the assumptions of bounded expectations. We have absorbed the second term in the bracket of \eqref{app:t1_expand} into $\up{C}$ since it does not depend on $\bftheta$. Now what we need to prove is 
\begin{align}
    -\mbb{E}_{p_\bfv}\mbb{E}_{p_d} [(\bfv^\intercal \bfs_m(\bfx; \bftheta))(\bfv^\intercal \bfs_d(\bfx; \bftheta))] = \mbb{E}_{p_\bfv}\mbb{E}_{p_d} [\bfv^\intercal \nabla_\bfx \bfs_m(\bfx;\bftheta) \bfv]. \label{app:t1_key}
\end{align}
This can be shown by first observing that
\begin{align}
    &-\mbb{E}_{p_\bfv}\mbb{E}_{p_d} [(\bfv^\intercal \bfs_m(\bfx; \bftheta))(\bfv^\intercal \bfs_d(\bfx; \bftheta))]\notag \\
    =& - \mbb{E}_{p_\bfv} \int p_d(\bfx) (\bfv^\intercal \bfs_m(\bfx; \bftheta))(\bfv^\intercal \bfs_d(\bfx; \bftheta)) \ud \bfx \notag\\
    =& -\mbb{E}_{p_\bfv} \int p_d(\bfx) (\bfv^\intercal \nabla_\bfx \log p_m(\bfx;\bftheta))(\bfv^\intercal \nabla_\bfx \log p_d(\bfx)) \ud \bfx \notag\\
    =& -\mbb{E}_{p_\bfv} \int (\bfv^\intercal \nabla_\bfx \log p_m(\bfx;\bftheta))(\bfv^\intercal \nabla_\bfx p_d(\bfx)) \ud \bfx \notag \\
    =& -\mbb{E}_{p_\bfv} \sum_{i=1}^D \int (\bfv^\intercal \nabla_\bfx \log p_m(\bfx;\bftheta)) v_i \frac{ \partial p_d(\bfx)}{\partial x_i} \ud \bfx \label{app:t1_mid},
\end{align}
where we assume $\bfx \in \mbb{R}^D$. Then, applying multivariate integration by parts (\cf, Lemma 4 in \citet{hyvarinen2005estimation}), we obtain
\begin{align*}
    & \bigg| \mbb{E}_{p_\bfv} \sum_{i=1}^D \int (\bfv^\intercal \bfs_m(\bfx; \bftheta)) v_i \frac{\partial p_d(\bfx)}{\partial x_i} \ud \bfx
    + \mbb{E}_{p_\bfv} \sum_{i=1}^D \int v_i p_d(\bfx) \bfv^\intercal \frac{\partial \bfs_m(\bfx; \bftheta)}{\partial x_i} \ud \bfx \bigg| \\
    =& \bigg| \mbb{E}_{p_\bfv} \bigg[\sum_{i=1}^D \lim_{x_i \rightarrow \infty} (\bfv^\intercal \bfs_m(\bfx; \bftheta)) v_i p_d(\bfx) - \sum_{i=1}^D \lim_{x_i \rightarrow -\infty} (\bfv^\intercal \bfs_m(\bfx; \bftheta)) v_i p_d(\bfx)\bigg] \bigg|\\
    \leq & \sum_{i=1}^D \lim_{x_i\rightarrow \infty} \sum_{j=1}^D \mbb{E}_{p_\bfv} |v_i v_j|| s_{m, j}(\bfx; \bftheta) p_d(\bfx)| + \sum_{i=1}^D \lim_{x_i\rightarrow -\infty} \sum_{j=1}^D \mbb{E}_{p_\bfv} |v_i v_j|| s_{m, j}(\bfx; \bftheta) p_d(\bfx)|\\
    \stackrel{(i)}{\leq} & \sum_{i=1}^D \lim_{x_i\rightarrow \infty} \sum_{j=1}^D \sqrt{\mbb{E}_{p_\bfv} v_i^2 \mbb{E}_{p_\bfv}v_j^2} |s_{m, j}(\bfx; \bftheta) p_d(\bfx)| + \sum_{i=1}^D \lim_{x_i\rightarrow -\infty} \sum_{j=1}^D \sqrt{\mbb{E}_{p_\bfv} v_i^2 \mbb{E}_{p_\bfv}v_j^2}| s_{m, j}(\bfx; \bftheta) p_d(\bfx)|\\
    \stackrel{(ii)}{=}& 0,
\end{align*}
where $s_{m,j}(\bfx; \bftheta)$ denotes the $j$-th component of $\bfs_m(\bfx; \bftheta)$. In the above derivation, $(i)$ is due to Cauchy-Schwarz inequality and $(ii)$ is from the assumption that $\mbb{E}_{p_\bfv}[\norm{\bfv}^2] < \infty$ and $s(\bfx;\bftheta) p_d(\bfx)$ vanishes at infinity.

Now returning to \eqref{app:t1_mid}, we have
\begin{align*}
    -\mbb{E}_{p_\bfv} \sum_{i=1}^D \int (\bfv^\intercal \nabla_\bfx \log p_m(\bfx;\bftheta)) v_i \frac{ \partial p_d(\bfx)}{\partial x_i} \ud \bfx &= \mbb{E}_{p_\bfv} \sum_{i=1}^D \int v_i p_d(\bfx) \bfv^\intercal \frac{\partial \bfs_m(\bfx; \bftheta)}{\partial x_i} \ud \bfx\\
    &=\mbb{E}_{p_\bfv} \int p_d(\bfx) \bfv^\intercal \nabla_\bfx \bfs_m(\bfx; \bftheta) \bfv \ud \bfx,
\end{align*}
which proves \eqref{app:t1_key} and the proof is completed.
\end{proof}

\begin{lemma}\label{app:thm:2}
Assume our model family is well-specified and identifiable (\assref{ass:identifiability}). Assume further that the densities are all positive (\assref{ass:positiveness}). When $p_\bfv$ satisfies some regularity conditions (\assref{ass:v}), we have
\begin{align*}
    L(\bftheta; p_\bfv) = 0 \Leftrightarrow \bftheta = \bftheta^*.
\end{align*}
\end{lemma}
\begin{proof}
First, since $p_d(\bfx) = p_m(\bfx; \bftheta^*) > 0$, $L(\bftheta; p_\bfv) = 0$ implies
\begin{align*}
    &\frac{1}{2} \mbb{E}_{p_\bfv} (\bfv^\intercal(\bfs_m(\bfx; \bftheta) - \bfs_d(\bfx)))^2 = 0\\
    \Leftrightarrow~& \mbb{E}_{p_\bfv} \bfv^\intercal (\bfs_m(\bfx; \bftheta) - \bfs_d(\bfx)) (\bfs_m(\bfx; \bftheta) - \bfs_d(\bfx))^\intercal \bfv = 0 \\
    \Leftrightarrow~& (\bfs_m(\bfx; \bftheta) - \bfs_d(\bfx))^\intercal \mbb{E}_{p_\bfv}[\bfv \bfv^\intercal] (\bfs_m(\bfx; \bftheta) - \bfs_d(\bfx)) = 0 \\
    \stackrel{(i)}{\Leftrightarrow}~& \bfs_m(\bfx; \bftheta) - \bfs_d(\bfx) = 0 \\
    \Leftrightarrow~& \log p_m(\bfx; \bftheta) = \log p_d(\bfx) + \up{C}
\end{align*}
where $(i)$ holds because $\mbb{E}_{p_\bfv}[\bfv \bfv^\intercal]$ is positive definite. Because $p_m(\bfx; \bftheta)$ and $p_d(\bfx)$ are normalized probability density functions, we have $p_m(\bfx; \bftheta) = p_d(\bfx)$. The identifiability assumption gives $\bftheta = \bftheta^*$. This concludes the left to right direction of the implication and the converse direction is trivial.
\end{proof}

\subsection{CONSISTENCY}\label{sec:consistency}
In addition to the assumptions in \thmref{app:thm:1} and \lemref{app:thm:2}, we need the following regularity conditions to prove the consistency of $\hat{\bftheta}_{N,M} \triangleq \argmin_{\bftheta\in\Theta} \hat{J}(\bftheta; \bfx_1^N, \bfv_{11}^{NM})$.
\begin{assumption}[Compactness]\label{ass:compact}
The parameter space $\Theta$ is compact.
\end{assumption}
\begin{assumption}[Lipschitz continuity]\label{ass:lipschitz}
Both $\nabla_\bfx \bfs_m(\bfx; \bftheta)$ and $\bfs_m(\bfx; \bftheta) \bfs_m(\bfx; \bftheta)^\intercal$ are Lipschitz continuous in terms of Frobenious norm, \ie, $\forall \bftheta_1\in \Theta, \bftheta_2 \in \Theta$, $\norm{\nabla_\bfx \bfs_m(\bfx; \bftheta_1) - \nabla_\bfx \bfs_m(\bfx; \bftheta_2)}_F \leq L_1(\bfx)\norm{\bftheta_1 - \bftheta_2}_2$ and $\norm{\bfs_m(\bfx; \bftheta_1) \bfs_m(\bfx; \bftheta_1)^\intercal - \bfs_m(\bfx; \bftheta_2) \bfs_m(\bfx; \bftheta_2)^\intercal}_F \leq L_2(\bfx)\norm{\bftheta_1 - \bftheta_2}_2$. In addition, we require $\mbb{E}_{p_d}[L_1^2(\bfx)] < \infty$ and $\mbb{E}_{p_d}[L_2^2(\bfx)] < \infty$.
\end{assumption}
\begin{assumption}[Bounded moments of projection vectors]\label{ass:v2}
$\mbb{E}_{p_\bfv}[\norm{\bfv \bfv^\intercal}_F^2] < \infty$.
\end{assumption}
\begin{lemma}\label{app:lemma:lipschitz}
Suppose $\bfs_m(\bfx; \bftheta)$ is sufficiently smooth (\assref{ass:lipschitz}) and $p_\bfv$ has bounded higher-order moments (\assref{ass:v2}). Let $f(\bftheta; \bfx, \bfv) \triangleq \bfv^\intercal \nabla_\bfx \bfs_m(\bfx; \bftheta) \bfv + \frac{1}{2}(\bfv^\intercal \bfs_m(\bfx; \bftheta))^2$. Then $f(\bftheta; \bfx, \bfv)$ is Lipschitz continuous with constant $L(\bfx, \bfv)$ and $\mbb{E}_{p_d, p_\bfv}[L^2(\bfx, \bfv)] < \infty$.
\end{lemma}
\begin{proof}
Let $A(\bftheta) \triangleq \nabla_\bfx \bfs_m(\bfx; \bftheta)$ and $B(\bftheta) \triangleq \bfs_m(\bfx; \bftheta) \bfs_m(\bfx; \bftheta)^\intercal$. Consider $\bftheta_1 \in \Theta, \bftheta_2 \in \Theta$ and let $D$ be the dimension of $\bfv$, we have
\begin{align*}
&|f(\bftheta_1; \bfx, \bfv) - f(\bftheta_2; \bfx, \bfv)|\\
=& \sum_{i=1}^D \sum_{j=1}^D \bigg[v_i v_j (A(\bftheta_1)_{ij}- A(\bftheta_2)_{ij}) + \frac{1}{2} v_i v_j (B(\bftheta_1)_{ij} -B(\bftheta_2)_{ij}) \bigg] \\
\overset{(i)}{\leq} & \sqrt{\sum_{i=1}^D\sum_{j=1}^D v_i^2 v_j^2 } \sqrt{\sum_{i=1}^D \sum_{j=1}^D \bigg[(A(\bftheta_1)_{ij}- A(\bftheta_2)_{ij}) + \frac{1}{2} (B(\bftheta_1)_{ij} -B(\bftheta_2)_{ij})\bigg]^2}\\
\overset{(ii)}{\leq} & \sqrt{\sum_{i=1}^D\sum_{j=1}^D v_i^2 v_j^2 } \sqrt{\sum_{i=1}^D \sum_{j=1}^D 2 (A(\bftheta_1)_{ij}- A(\bftheta_2)_{ij})^2 + \frac{1}{2} (B(\bftheta_1)_{ij} -B(\bftheta_2)_{ij})^2}\\
\overset{(iii)}{\leq} & \sqrt{\sum_{i=1}^D\sum_{j=1}^D v_i^2 v_j^2 } \sqrt{2 L_1^2(\bfx) \norm{\bftheta_1 - \bftheta_2}^2_2 + \frac{1}{2} L_2^2(\bfx) \norm{\bftheta_1 - \bftheta_2}^2_2}\\
= & \sqrt{\sum_{i=1}^D\sum_{j=1}^D v_i^2 v_j^2 } \sqrt{2 L_1^2(\bfx) + \frac{1}{2} L_2^2(\bfx)} \norm{\bftheta_1 - \bftheta_2}_2,
\end{align*}
where $(i)$ is Cauchy-Schwarz inequality, $(ii)$ is Jensen's inequality, and $(iii)$ is due to \assref{ass:lipschitz}. Now let 
\begin{align*}
    L(\bfx, \bfv) \triangleq \sqrt{\sum_{i=1}^D\sum_{j=1}^D v_i^2 v_j^2 } \sqrt{2 L_1^2(\bfx) + \frac{1}{2} L_2^2(\bfx)}.
\end{align*}
Then
\begin{align*}
    \mbb{E}_{p_d, p_\bfv}[L^2(\bfx, \bfv)] &\overset{(i)}{=} \mbb{E}_{p_\bfv}\bigg[\sum_{i=1}^D\sum_{j=1}^D v_i^2 v_j^2 \bigg] \mbb{E}_{p_d}\bigg[2 L_1^2(\bfx) + \frac{1}{2} L_2^2(\bfx)\bigg]\\
    &\overset{(ii)}{<} \infty,
\end{align*}
where $(i)$ results from the independence of $\bfv, \bfx$, and $(ii)$ is due to  \assref{ass:lipschitz} and \assref{ass:v2}.
\end{proof}
\begin{lemma}[Uniform convergence of the expected error]\label{app:lemma:1}
Under \assref{ass:compact}-\ref{ass:v2}, we have
\begin{align}
    \mbb{E}_{p_\bfv, p_d}\bigg[\sup_{\bftheta \in \Theta} \big|\hat{J}(\bftheta; \bfx_1^N, \bfv_{11}^{NM}) - J(\bftheta; p_\bfv) \big| \bigg] \leq O\left(\operatorname{diam}(\Theta) \sqrt{\frac{D}{N}} \right)
\end{align}
where $\operatorname{diam}(\cdot)$ denotes the diameter and $D$ is the dimension of $\Theta$.
\end{lemma}
\begin{proof}
The proof consists of 3 steps. First, we use the symmetrization trick to get rid of the inner expectation term $J(\bftheta; p_\bfv) = \mbb{E}_{p_\bfv, p_d}[\hat{J}(\bftheta; \bfx_1^N, \bfv_{11}^{NM})]$. Second, we use chaining to get an upper bound that involves integration of the metric entropy. Finally, we upper bound the metric entropy to obtain the uniform convergence bound.

\textbf{Step 1:} From Jensen's inequality, we obtain
\begin{align*}
    & \mbb{E}_{p_\bfv, p_d}\big[\sup_{\bftheta \in \Theta} \big|\hat{J}(\bftheta; \bfx_1^N, \bfv_{11}^{NM}) - J(\bftheta; p_\bfv) \big| \big]\\
    = & \mbb{E}_{p_\bfv, p_d}\big[\sup_{\bftheta \in \Theta} \big|\hat{J}(\bftheta; \bfx_1^N, \bfv_{11}^{NM}) - \mbb{E}_{p_\bfv, p_d}[\hat{J}(\bftheta; \bfx_1^N, \bfv_{11}^{NM})]\big| \big]\\
    = & \mbb{E}_{p_\bfv, p_d}\big[\sup_{\bftheta \in \Theta} \big|\hat{J}(\bftheta; \bfx_1^N, \bfv_{11}^{NM}) - \mbb{E}_{p_\bfv, p_d}[\hat{J}(\bftheta; {\bfx'}_1^N, {\bfv'}_{11}^{NM}]\big| \big]\\
    \leq & \mbb{E}_{p_\bfv, p_d}\big[\sup_{\bftheta \in \Theta}\big|\hat{J}(\bftheta; \bfx_1^N, \bfv_{11}^{NM}) - \hat{J}(\bftheta; {\bfx'}_1^N, {\bfv'}_{11}^{NM})\big| \big],
\end{align*}
where ${\bfx'}_1^N, {\bfv'}_{11}^{NM}$ are independent copies of $\bfx_1^N$ and $\bfv_{11}^{NM}$. Let $\{\epsilon_{i}\}_{i=1}^N$ be a set of independent Rademacher random variables, we have
\begin{align}
    &\mbb{E}_{p_\bfv, p_d}\big[\sup_{\bftheta \in \Theta}\big|\hat{J}(\bftheta; \bfx_1^N, \bfv_{11}^{NM}) - \hat{J}(\bftheta; {\bfx'}_1^N, {\bfv'}_{11}^{NM})\big| \big]\notag\\
    =&\mbb{E}_{p_\bfv, p_d}\bigg[ \sup_{\bftheta \in \Theta} \frac{1}{N} \frac{1}{M} \bigg|\sum_{i=1}^N \sum_{j=1}^M \underbrace{\bfv_{ij}^\intercal \nabla_\bfx \bfs_m(\bfx_i; \bftheta) \bfv_{ij} 
    + \frac{1}{2} \left(\bfv_{ij}^\intercal \bfs_m(\bfx_i; \bftheta) \right)^2}_{\triangleq f(\bftheta; \bfx_i, \bfv_{ij})} - \bigg[ \underbrace{{\bfv'}_{ij}^\intercal \nabla_\bfx \bfs_m(\bfx'_i; \bftheta) {\bfv'}_{ij} + \frac{1}{2} \left({\bfv'}_{ij}^\intercal \bfs_m(\bfx'_i; \bftheta) \right)^2}_{\triangleq f(\bftheta; \bfx'_i, \bfv'_{ij})} \bigg] \bigg| \bigg]\notag\\
    =& \mbb{E}_{p_\bfv, p_d}\bigg[  \sup_{\bftheta \in \Theta} \frac{1}{N}\frac{1}{M} \bigg|\sum_{i=1}^N \sum_{j=1}^M f(\bftheta; \bfx_i, \bfv_{ij}) - f(\bftheta; \bfx'_i, \bfv'_{ij}) \bigg| \bigg]\notag\\
    \stackrel{(i)}{=}& \mbb{E}\bigg[\sup_{\bftheta \in \Theta} \bigg|\frac{1}{N}  \frac{1}{M} \sum_{i=1}^N \sum_{j=1}^M  \epsilon_{i}( f(\bftheta; \bfx_i, \bfv_{ij}) - f(\bftheta; \bfx'_i, \bfv'_{ij})) \bigg| \bigg]\notag \\
    \stackrel{(ii)}{\leq}& \mbb{E}\bigg[\sup_{\bftheta \in \Theta} \bigg| \frac{1}{N} \frac{1}{M} \sum_{i=1}^N \sum_{j=1}^M  \epsilon_{i} f(\bftheta; \bfx_i, \bfv_{ij}) \bigg| \bigg] + \mbb{E}\bigg[\sup_{\bftheta \in \Theta} \bigg|\frac{1}{N} \frac{1}{M}  \sum_{i=1}^N \sum_{j=1}^M  \epsilon_{i} f(\bftheta; \bfx'_i, \bfv'_{ij}) \bigg| \bigg]\notag \\
    =& 2\mbb{E}\bigg[\sup_{\bftheta \in \Theta} \bigg| \frac{1}{N} \frac{1}{M} \sum_{i=1}^N \sum_{j=1}^M  \epsilon_{i} f(\bftheta; \bfx_i, \bfv_{ij}) \bigg| \bigg]
    \label{app:eqn:s1},
\end{align}
where $(i)$ is because the quantity is symmetric about $0$ in distribution, and $(ii)$ is due to Jensen's inequality. 

\textbf{Step 2:} First note that given $\bfx_i, \bfv_{ij}$, $\frac{\epsilon_{i}}{M}\sum_{j=1}^M f(\bftheta; \bfx_i, \bfv_{ij})$ is a zero-mean sub-Gaussian process \wrt $\bftheta$. This can be observed from
\begin{align*}
    &\mbb{E}_{\epsilon_{i}}\left[ e^{\lambda \frac{\epsilon_{i}}{M}\sum_{j=1}^M [f(\bftheta_1; \bfx_i, \bfv_{ij}) - f(\bftheta_2; \bfx_i, \bfv_{ij})]} \right]\\
    \stackrel{(i)}{\leq}& \exp\left\{\frac{\lambda^2}{2M^2}\bigg(\sum_{j=1}^M [f(\bftheta_1; \bfx_i, \bfv_{ij}) - f(\bftheta_2; \bfx_i, \bfv_{ij})]\bigg)^2\right\}\\
    \stackrel{(ii)}{\leq} & \exp\left\{\frac{\lambda^2}{2M^2}M \sum_{j=1}^M[f(\bftheta_1; \bfx_i, \bfv_{ij}) - f(\bftheta_2; \bfx_i, \bfv_{ij})]^2\right\}\\
    \stackrel{(iii)}{\leq}& \exp\left\{ \frac{\lambda^2}{2M}\sum_{j=1}^M L^2(\bfx_i, \bfv_{ij})\norm{\bftheta_1 - \bftheta_2}^2 \right\},
\end{align*}
where $(i)$ holds because $\epsilon_{i}$ is a 1-sub-Gaussian random variable, $(ii)$ is from Cauchy-Schwarz inequality and $(iii)$ is due to \lemref{app:lemma:lipschitz}. As a result, $\frac{1}{N} \frac{1}{M}  \sum_{i=1}^N \sum_{j=1}^M \epsilon_{i} f(\bftheta; \bfx_i, \bfv_{ij})$ is a zero-mean sub-Gaussian random process with metric
\begin{align*}
    d(\bftheta_1, \bftheta_2) = \frac{1}{\sqrt{N}} \sqrt{\frac{1}{NM}\sum_{i=1}^N \sum_{j=1}^M L^2(\bfx_i, \bfv_{ij})} \norm{\bftheta_1 - \bftheta_2}.
\end{align*}
Since $\Theta$ is compact, the diameter of $\Theta$ with respect to the Euclidean norm $\norm{\cdot}_2$ is finite and we denote it as $\operatorname{diam}(\Theta) < \infty$. Then, Dudley's entropy integral~\citep{dudley1967sizes} gives
\begin{align}
    \mbb{E}\bigg[\sup_{\bftheta \in \Theta} \bigg| \frac{1}{N} \frac{1}{M} \sum_{i=1}^N \sum_{j=1}^M  \epsilon_{i} f(\bftheta; \bfx_i, \bfv_{ij}) \bigg| \bigg] \leq O(1)\mbb{E}\bigg[ \int_0^{ \frac{1}{\sqrt{N}} \sqrt{\frac{1}{NM}\sum_{i=1}^N \sum_{j=1}^M L^2(\bfx_i, \bfv_{ij})}  \operatorname{diam}(\Theta)} \sqrt{\log N(\Theta, d, \epsilon)} \ud \epsilon\bigg].\label{app:eqn:s2}
\end{align}
Here $\log N(\Theta, d, \epsilon)$ is the metric entropy of $\Theta$ with metric $d(\bftheta_1, \bftheta_2) =  \frac{1}{\sqrt{N}} \sqrt{\frac{1}{NM}\sum_{i=1}^N \sum_{j=1}^M L^2(\bfx_i, \bfv_{ij})}  \norm{\bftheta_1 - \bftheta_2}_2$ and size $\epsilon$. 

\textbf{Step 3:} When the dimension of $\bftheta \in \Theta$ is $D$, it is known that the $\epsilon$-covering number of $\Theta$ with Euclidean distance is
\begin{align*}
    N(\Theta, \norm{\cdot}, \epsilon) \leq \bigg(1 + \frac{\operatorname{diam}(\Theta)}{\epsilon} \bigg)^D.
\end{align*}
Therefore, $N(\Theta, d, \epsilon)$ can be bounded by
\begin{align*}
    N(\Theta, d, \epsilon) \leq  \left(1 + \sqrt{\frac{1}{NM}\sum_{i=1}^N \sum_{j=1}^M L^2(\bfx_i, \bfv_{ij})} \frac{\operatorname{diam}(\Theta)}{\sqrt{N}\epsilon} \right)^D.
\end{align*}
Hence, the metric integral can be bounded
\begin{align}
    &\int_0^{ \frac{1}{\sqrt{N}} \sqrt{\frac{1}{NM}\sum_{i=1}^N \sum_{j=1}^M L^2(\bfx_i, \bfv_{ij})}  \operatorname{diam}(\Theta)} \sqrt{\log N(\Theta, d, \epsilon)} \ud \epsilon \notag\\
    &\leq \int_0^{ \frac{1}{\sqrt{N}} \sqrt{\frac{1}{NM}\sum_{i=1}^N \sum_{j=1}^M L^2(\bfx_i, \bfv_{ij})}  \operatorname{diam}(\Theta)} \sqrt{D \log \Bigg(1 + \sqrt{\frac{1}{NM}\sum_{i=1}^N \sum_{j=1}^M L^2(\bfx_i, \bfv_{ij})} \frac{\operatorname{diam}(\Theta)}{\sqrt{N}\epsilon}  \Bigg)} \ud \epsilon \notag\\
    &\leq \int_0^{\frac{1}{\sqrt{N}} \sqrt{\frac{1}{NM}\sum_{i=1}^N \sum_{j=1}^M L^2(\bfx_i, \bfv_{ij})}  \operatorname{diam}(\Theta)} \sqrt{\sqrt{\frac{1}{NM}\sum_{i=1}^N \sum_{j=1}^M L^2(\bfx_i, \bfv_{ij})} \frac{D\operatorname{diam}(\Theta)}{\sqrt{N}\epsilon}}  \ud \epsilon\notag \\
    &= 2 \sqrt{\frac{1}{NM}\sum_{i=1}^N \sum_{j=1}^M L^2(\bfx_i, \bfv_{ij})} \sqrt{\frac{D}{N}} \operatorname{diam}(\Theta) \label{app:eqn:s3}
\end{align}

Finally, combining \eqref{app:eqn:s1}, \eqref{app:eqn:s2} and \eqref{app:eqn:s3} gives us
\begin{align*}
    & \mbb{E}_{p_\bfv, p_d}\bigg[\sup_{\bftheta \in \Theta} \big|\hat{J}(\bftheta; \bfx_1^N, \bfv_{11}^{NM}) - J(\bftheta; p_\bfv) \big| \bigg] \\
    \leq& 4 O(1) \mbb{E}_{p_\bfv, p_d} \left[\sqrt{\frac{1}{NM}\sum_{i=1}^N \sum_{j=1}^M L^2(\bfx_i, \bfv_{ij})} \sqrt{\frac{D}{N}} \operatorname{diam}(\Theta) \right]\\
    \stackrel{(i)}{\leq} & O(1)\sqrt{\frac{D}{N}} \operatorname{diam}(\Theta) \sqrt{\mbb{E}_{p_\bfv, p_d}\bigg[\frac{1}{NM}\sum_{i=1}^N \sum_{j=1}^M L^2(\bfx_i, \bfv_{ij})\bigg]}\\
    \stackrel{(ii)}{\leq} & O(1)\operatorname{diam}(\Theta)\sqrt{\frac{D}{N}},
\end{align*}
where $(i)$ is due to Jensen's inequality and $(ii)$ results from $\mbb{E}[L^2(\bfx, \bfv)] < \infty$, and the compactness of $\Theta$ guarantees that the bound is finite. 
\end{proof}
\begin{customthm}{\ref{app:thm:consistency}}
Suppose all the previous assumptions hold (\assref{ass:score}-\assref{ass:v2}). Assume further the conditions of \thmref{app:thm:1} and \lemref{app:thm:2} are satisfied. Let $\bftheta^*$ be the true parameter of the data distribution, and $\hat{\bftheta}_{N, M}$ be the empirical estimator defined by
\begin{align*}
    \hat{\bftheta}_{N, M} \triangleq \argmin_{\bftheta \in \Theta} \hat{J}(\bftheta; \bfx_1^N, \bfv_{11}^{NM})
\end{align*}
Then, $\hat{\bftheta}_{N,M}$ is consistent, meaning that
\begin{align*}
    \hat{\bftheta}_{N, M} \overset{p}{\to} \bftheta^*
\end{align*}
as $N\to \infty$.
\end{customthm}
\begin{proof}
Note that \thmref{app:thm:1} and \lemref{app:thm:2} together imply that $\bftheta^* = \argmin_{\bftheta \in \Theta} J(\bftheta; p_\bfv)$. Then, we will show $J(\hat{\bftheta}_{N, M}; p_\bfv) \overset{p}{\to} J(\bftheta^*; p_\bfv)$ when $N \to \infty$. This can be done by noticing
\begin{align}
    J(\hat{\bftheta}_{N, M}; p_\bfv) - J(\bftheta^*; p_\bfv) =& J(\hat{\bftheta}_{N, M}; p_\bfv) - \hat{J}(\hat{\bftheta}_{N, M}; \bfx_1^N, \bfv_{11}^{NM}) + \hat{J}(\hat{\bftheta}_{N, M}; \bfx_1^N, \bfv_{11}^{NM}) - \hat{J}(\bftheta^*; \bfx_1^N, \bfv_{11}^{NM}) \notag \\
    & + \hat{J}(\bftheta^*; \bfx_1^N, \bfv_{11}^{NM}) - J(\bftheta^*; p_\bfv)\notag \\
    \leq & \sup_{\bftheta \in \Theta} |\hat{J}(\bftheta; \bfx_1^N, \bfv_{11}^{NM}) - J(\bftheta; p_\bfv)| + |\hat{J}(\bftheta^*; \bfx_1^N, \bfv_{11}^{NM}) - J(\bftheta^*; p_\bfv)|\notag\\
    \leq & 2\sup_{\bftheta \in \Theta} |\hat{J}(\bftheta; \bfx_1^N, \bfv_{11}^{NM}) - J(\bftheta; p_\bfv)|
    \label{app:eqn:c1}
\end{align} 
We can easily conclude that \eqref{app:eqn:c1} is $o_p(1)$ with the help of Lemma~\ref{app:lemma:1}, because
\begin{align*}
    P\bigg(\sup_{\bftheta \in \Theta} |\hat{J}(\bftheta; \bfx_1^N, \bfv_{11}^{NM}) - J(\bftheta; p_\bfv)| > t \bigg) \leq \mbb{E}\bigg[\sup_{\bftheta \in \Theta} |\hat{J}(\bftheta; \bfx_1^N, \bfv_{11}^{NM}) - J(\bftheta; p_\bfv)||\bigg] / t \leq O(1)\sqrt{\frac{1}{N t^2}} \to 0,
\end{align*}
as $N \to \infty$. From \lemref{app:thm:2} we also have $L(\hat{\bftheta}_{N, M}; p_\bfv) - L(\bftheta^*; p_\bfv) > 0$ if $\hat{\bftheta}_{N, M} \neq \bftheta$. As shown by \thmref{app:thm:1}, this is the same as $J(\hat{\bftheta}_{N, M}; p_\bfv) - J(\bftheta^*; p_\bfv) > 0$ if $\hat{\bftheta}_{N, M} \neq \bftheta$. Therefore \eqref{app:eqn:c1} = $o_p(1)$ gives $J(\hat{\bftheta}_{N, M}; p_\bfv) \overset{p}{\to} J(\bftheta^*; p_\bfv)$. 

Next, we show $\hat{\bftheta}_{N,M} \overset{p}{\to} \bftheta^*$. This can be inferred from $J(\hat{\bftheta}_{N, M}; p_\bfv) \overset{p}{\to} J(\bftheta^*; p_\bfv)$ because $J(\bftheta; p_\bfv)$ is continuous (Assumption~\ref{ass:lipschitz}) and $\Theta$ is compact (Assumption~\ref{ass:compact}). The proof is reductio ad absurdum. Specifically, assume that $\hat{\bftheta}_{N,M} \not \overset{p}{\to} \bftheta^*$. We know $\exists \epsilon > 0, \delta > 0, \forall K > 0, \exists N > K, M > 0$ such that $ P(\|\hat{\bftheta}_{N,M}-\bftheta^*\| \geq \epsilon) \geq \delta$. Note that $J(\bftheta; p_\bfv) = \mbb{E}[f(\bftheta; \bfx, \bfv)]$, $f(\bftheta; \bfx, \bfv)$ is Lipschitz continuous and $\mbb{E}[L(\bftheta; \bfx, \bfv)] \leq \sqrt{\mbb{E}[L^2(\bftheta; \bfx, \bfv)]} < \infty$. This implies that $J(\bftheta; p_\bfv)$ is continuous \wrt $\bftheta$. Since $\Theta$ is compact and $J(\bftheta; p_\bfv)$ is continuous, we can define a compact set $\mcal{S}_\epsilon \triangleq \{\bftheta \in \Theta | \| \bftheta - \bftheta^* \| \geq \epsilon\}$ and let $\bftheta_{\mcal{S}_\epsilon} \triangleq \argmin_{\bftheta \in \mcal{S}_\epsilon} J(\bftheta; p_\bfv)$. Observe that
\begin{align*}
    p(J(\hat{\bftheta}_{N,M}; p_\bfv) \geq J(\bftheta_{\mcal{S}_\epsilon}; p_\bfv)) &= p(|J(\hat{\bftheta}_{N,M}; p_\bfv) - J(\bftheta^*;p_\bfv)| \geq J(\bftheta_{\mcal{S}_\epsilon}; p_\bfv) - J(\bftheta^*; p_\bfv))\\
    &\geq p(\|\hat{\bftheta}_{N,M}-\bftheta^*\| \geq \epsilon) \geq \delta.
\end{align*}
However, the fact that $p(|J(\hat{\bftheta}_{N,M}; p_\bfv) - J(\bftheta^*;p_\bfv)| \geq J(\bftheta_{\mcal{S}_\epsilon}; p_\bfv) - J(\bftheta^*; p_\bfv)) \geq \delta$ holds for arbitrarily large $N$ contradicts $J(\hat{\bftheta}_{N, M}; p_\bfv) \overset{p}{\to} J(\bftheta^*; p_\bfv)$.
\end{proof}

\subsection{ASYMPTOTIC NORMALITY}\label{app:sec:normality}
\paragraph{Notations.} To simplify notations we use $\partial_i \partial_j h(\cdot) \triangleq (\nabla_\bfx^2 h(\cdot))_{ij}$, $\partial_i h(\cdot) \triangleq (\nabla_\bfx h(\cdot))_i$, and denote $\nabla_\bfx h(\bfx)|_{\bfx = \bfx'}$ as $\nabla_\bfx h(\bfx')$. Here $h(\cdot)$ denotes some arbitrary function. Let $l_m \triangleq \log p_m(\bfx; \bftheta)$, $l_m(\bfx; \bftheta)\triangleq \log p_m(\bfx; \bftheta)$ and further adopt the following notations
\begin{align*}
J(\bftheta) &\triangleq \mbb{E}_{p_d}\bigg[\operatorname{tr}(\nabla_{\bfx} \bfs_m(\bfx; \bftheta)) + \frac{1}{2}\norm{ \bfs_m(\bfx; \bftheta)}_2^2\bigg]\\
f(\bftheta; \bfx, \bfv) &\triangleq \bfv^\intercal \nabla_\bfx \bfs_m(\bfx; \bftheta) \bfv + \frac{1}{2} (\bfv^\intercal \bfs_m(\bfx; \bftheta))^2\\
f(\bftheta; \bfx, \bfv_1^M) &\triangleq \frac{1}{M}\sum_{j=1}^Mf(\bftheta; \bfx, \bfv_j) = \frac{1}{M}\sum_{j=1}^M \bfv_j^\intercal \nabla_\bfx \bfs_m(\bfx; \bftheta) \bfv_j + \frac{1}{2} (\bfv_j^\intercal \bfs_m(\bfx; \bftheta))^2\\
f(\bftheta; \bfx) &\triangleq \operatorname{tr}(\nabla_\bfx \bfs_m(\bfx; \bftheta)) + \frac{1}{2} \norm{\bfs_m(\bfx; \bftheta)}_2^2\\
\Sigma_{ij} &\triangleq (\mbb{E}_{p_\bfv}[\bfv\bfv^\intercal])_{ij}\\
\mathfrak{S}_{ijpq} &\triangleq \mbb{E}_{p_\bfv}[v_i v_j v_p v_q]\\
\mathfrak{V}_{ijpq} &\triangleq \mbb{E}_{p_d}\left[ \left( \nabla_\bftheta\partial_i\partial_j l_m + \frac{1}{2} \nabla_\bftheta( \partial_i l_m \partial_j l_m)\right)\left(\nabla_\bftheta \partial_p\partial_q  l_m + \frac{1}{2} \nabla_\bftheta( \partial_p l_m \partial_q l_m)\right)^\intercal \right]\bigg|_{\bftheta = \bftheta^*}\\
V_{ij} &\triangleq \mathfrak{V}_{iijj}\\
W_{ij} &\triangleq \mathfrak{V}_{ijij} 
\end{align*}

For the proof of asymptotic normality we need the following extra assumptions.
\begin{assumption}[Lipschitz smoothness on second derivatives]\label{ass:lipschitz2}
For $\bftheta_1$, $\bftheta_2$ near $\bftheta^*$, and $\forall i,j$,
\begin{align*}
    \norm{\nabla_\bftheta^2 \partial_i\partial_j l_m(\bfx; \bftheta_1) - \nabla_\bftheta^2 \partial_i\partial_j l_m(\bfx; \bftheta_2)}_F &\leq M_{ij}(\bfx) \norm{\bftheta_1 - \bftheta_2}_2\\
    \norm{\nabla_\bftheta^2 \partial_i l_m(\bfx; \bftheta_1) \partial_j l_m(\bfx; \bftheta_1) - \nabla_\bftheta^2 \partial_i l_m(\bfx; \bftheta_2) \partial_j l_m(\bfx; \bftheta_2)}_F &\leq N_{ij}(\bfx) \norm{\bftheta_1 - \bftheta_2}_2
\end{align*}
and
\begin{align*}
    \mbb{E}_{p_d}[M_{ij}^2(\bfx)] < \infty, \quad \mbb{E}_{p_d}[N_{ij}^2(\bfx)] < \infty, \quad \forall i,j.
\end{align*}
\end{assumption}

\begin{lemma}\label{app:lemma:derivative}
Suppose $l_m(\bfx; \bftheta)$ is sufficiently smooth (\assref{ass:lipschitz2}) and $p_\bfv$ has bounded moments (\assref{ass:v} and \assref{ass:v2}). Let $\nabla_\bftheta^2 f(\bftheta; \bfx, \bfv_1^M) \triangleq \frac{1}{M}\sum_{i=1}^M \nabla_\bftheta^2 \bfv_i^\intercal \nabla_\bfx \bfs_m(\bfx; \bftheta) \bfv_i + \frac{1}{2} \nabla_\bftheta^2 (\bfv_i^\intercal \bfs_m(\bfx; \bftheta))^2$. Then $\nabla_\bftheta^2f(\bftheta; \bfx, \bfv)$ is Lipschitz continuous, \ie, for $\bftheta_1$ and $\bftheta_2$ close to $\bftheta^*$, there exists a Lipschitz constant $L(\bfx, \bfv_1^M)$ such that
\begin{align*}
    \norm{\nabla_\bftheta^2f(\bftheta_1; \bfx, \bfv_1^M) - \nabla_\bftheta^2f(\bftheta_2; \bfx, \bfv_1^M)}_F \leq L(\bfx, \bfv_1^M) \norm{\bftheta_1 - \bftheta_2}_2, 
\end{align*}
and $\mbb{E}_{p_d,p_\bfv}[L^2(\bfx, \bfv_1^M)] < \infty$.
\end{lemma}
\begin{proof}
First, we write out $\nabla_\bftheta^2f(\bftheta_1; \bfx, \bfv_1^M) - \nabla_\bftheta^2f(\bftheta_2; \bfx, \bfv_1^M)$ according to the definitions. 
Let $A_{ij}(\bftheta) \triangleq \nabla_\bftheta^2 \partial_i \partial_j l_m(\bfx;\bftheta)$ and $B_{ij}(\bftheta) \triangleq \nabla_\bftheta^2 \partial_i l_m(\bfx;\bftheta)\partial_jl_m(\bfx;\bftheta)$. Then,
\begin{align*}
    \nabla_\bftheta^2f(\bftheta_1; \bfx, \bfv_1^M) - \nabla_\bftheta^2f(\bftheta_2; \bfx, \bfv_1^M) = \frac{1}{M} \sum_{i,j,k} v_{k,i}v_{k,j} \bigg[ A_{ij}(\bftheta_1) - A_{ij}(\bftheta_2) + \frac{1}{2}(B_{ij}(\bftheta_1) - B_{ij}(\bftheta_2)) \bigg].
\end{align*}
Then, Cauchy-Schwarz and Jensen's inequality give
\begin{align*}
    &\norm{\nabla_\bftheta^2f(\bftheta_1; \bfx, \bfv_1^M) - \nabla_\bftheta^2f(\bftheta_2; \bfx, \bfv_1^M)}_F^2\\
    =& \sum_{l,m}\bigg(\frac{1}{M} \sum_{i,j,k} v_{k,i}v_{k,j} \bigg[ A_{ij}(\bftheta_1)_{lm} - A_{ij}(\bftheta_2)_{lm} + \frac{1}{2}(B_{ij}(\bftheta_1)_{lm} - B_{ij}(\bftheta_2))_{lm} \bigg]\bigg)^2\\
    \leq& \sum_{l,m}\bBigg@{4}(\frac{1}{M} \sqrt{\sum_{i,j} \bigg(\sum_{k}v_{k,i}v_{k,j}\bigg)^2} \cdot \sqrt{\sum_{i,j}\bigg[ A_{ij}(\bftheta_1)_{lm} - A_{ij}(\bftheta_2)_{lm} + \frac{1}{2}(B_{ij}(\bftheta_1)_{lm} - B_{ij}(\bftheta_2))_{lm} \bigg]^2}\bBigg@{4})^2\\
    \leq& \sum_{l,m}\frac{1}{M^2} \Bigg(\sum_{i,j} \bigg(\sum_{k}v_{k,i}v_{k,j}\bigg)^2\Bigg) \Bigg(\sum_{i,j}\bigg(2\bigg[ A_{ij}(\bftheta_1)_{lm} - A_{ij}(\bftheta_2)_{lm}\bigg]^2 + \frac{1}{2}\bigg[(B_{ij}(\bftheta_1)_{lm} - B_{ij}(\bftheta_2))_{lm} \bigg]^2\bigg)\Bigg)\\
    =& \frac{1}{M^2} \Bigg(\sum_{i,j} \bigg(\sum_{k}v_{k,i}v_{k,j}\bigg)^2\Bigg) \Bigg(\sum_{i,j}\bigg(\sum_{l,m}2\bigg[ A_{ij}(\bftheta_1)_{lm} - A_{ij}(\bftheta_2)_{lm}\bigg]^2 + \sum_{l,m}\frac{1}{2}\bigg[(B_{ij}(\bftheta_1)_{lm} - B_{ij}(\bftheta_2))_{lm} \bigg]^2\bigg)\Bigg)\\
    =& \frac{1}{M^2} \Bigg(\sum_{i,j,p,q}v_{p,i}v_{p,j}v_{q,i}v_{q,j}\Bigg) \Bigg(\sum_{i,j}\bigg(2\norm{A_{ij}(\bftheta_1) - A_{ij}(\bftheta_2)}_F^2 + \frac{1}{2}\norm{V_{ij}(\bftheta_1) - V_{ij}(\bftheta_2)}_F^2\bigg)\Bigg)\\
    \leq& \underbrace{\frac{1}{M^2} \bigg(\sum_{i,j,p,q}v_{p,i}v_{p,j}v_{q,i}v_{q,j}\bigg) \bigg(\sum_{i,j}\bigg(2M_{ij}^2 + \frac{1}{2}N_{ij}^2\bigg)\bigg)}_{\triangleq L^2(\bfx, \bfv_1^M)}\norm{\bftheta_1 - \bftheta_2}_2^2
\end{align*}
Next, we bound the expectation
\begin{align*}
    \mbb{E}_{p_d, p_\bfv}[L^2(\bfx, \bfv_1^M)] &= \frac{1}{M^2} \mbb{E}_{p_\bfv} \bigg[\sum_{i,j,p,q} v_{p,i}v_{p,j}v_{q,i}v_{q,j} \bigg] \mbb{E}_{p_d}\bigg[\sum_{ij} 2M_{ij}^2(\bfx) + \frac{1}{2}N_{ij}^2(\bfx)\bigg]\\
    &\overset{(i)}{\leq} O(1) \mbb{E}_{p_\bfv} \bigg[\sum_{i,j,p,q} v_{p,i}v_{p,j}v_{q,i}v_{q,j} \bigg]\\
    &= O(1) \Big(\sum_{ij} M(M-1) \mbb{E}_{p_\bfv}[v_iv_j]^2 + M\mbb{E}_{p_\bfv}[(v_iv_j)^2]\Big)\\
    &= O(1) \Big(M(M-1) \norm{\mbb{E}_{p_\bfv}[\bfv \bfv^\intercal]}_F^2 + M \mbb{E}_{p_\bfv}[\norm{\bfv \bfv^\intercal}_F^2\Big)\\
    &\overset{(ii)}{\leq} O(1) \Big(M(M-1) \mbb{E}_{p_\bfv}[\norm{\bfv \bfv^\intercal}_F^2] + M \mbb{E}_{p_\bfv}[\norm{\bfv \bfv^\intercal}_F^2\Big)
    \overset{(iii)}{<} \infty,
\end{align*}
where $(i)$ is due to \assref{ass:lipschitz2}, $(ii)$ is Jensen's, and $(iii)$ is because of \assref{ass:v} and \ref{ass:v2}. 
\end{proof}
\begin{lemma}\label{app:lemma:variance}
Assume that conditions in \thmref{app:thm:1} and \lemref{app:thm:2} hold, and $p_\bfv$ has bounded higher-order moments (\assref{ass:v2}). 
Then
\begin{align}
    \operatorname{Var}_{p_d, p_\bfv} \left[\nabla_\bftheta f(\bftheta^*; \bfx, \bfv_1^M) \right] = \sum_{i,j,p,q}\bigg[\bigg( 1 - \frac{1}{M}\bigg) \Sigma_{ij}\Sigma_{pq} + \frac{1}{M}\mathfrak{S}_{ijpq}\bigg] \mathfrak{V}_{ijpq},\label{app:eqn:var}
\end{align}
where $\nabla_\bftheta f(\bftheta^*; \bfx, \bfv_1^M) = \nabla_\bftheta f(\bftheta; \bfx, \bfv_1^M)\big|_{\bftheta = \bftheta^*}$ 

In particular, if $p_\bfv \sim \mcal{N}(0, I)$, we have
\begin{align*}
    \operatorname{Var}_{p_d, p_\bfv} \left[\nabla_\bftheta f(\bftheta^*; \bfx, \bfv_1^M) \right] = \sum_{ij}V_{ij} + \frac{2}{M}\sum_{i}V_{ii} +  \frac{2}{M}\sum_{i\neq j}W_{ij}.
\end{align*}
If $p_\bfv$ is the distribution of multivariate Rademacher random variables, we have
\begin{align*}
    \operatorname{Var}_{p_d, p_\bfv} \left[\nabla_\bftheta f(\bftheta^*; \bfx, \bfv_1^M) \right] = \sum_{ij}V_{ij}+  \frac{2}{M}\sum_{i\neq j}W_{ij}.
\end{align*}
\end{lemma}
\begin{proof}
 Since $\bftheta^*$ is the true parameter of the data distribution, we have
\begin{align*}
\mbb{E}_{p_d, p_\bfv}[\nabla_\bftheta f(\bftheta^*; \bfx, \bfv_1^M)] = \nabla_\bftheta \mbb{E}_{p_d, p_\bfv}[f(\bftheta^*; \bfx, \bfv_1^M)] = \nabla_\bftheta J(\bftheta^*; p_\bfv) = 0.
\end{align*}
Therefore, \eqref{app:eqn:var} can be expanded as
\begin{align*}
    &\operatorname{Var}_{p_d, p_\bfv} \left[\nabla_\bftheta f(\bftheta^*; \bfx, \bfv_1^M) \right]\\
    =& \mbb{E}_{p_d, p_\bfv} \left[\nabla_\bftheta f(\bftheta^*; \bfx, \bfv_1^M)\nabla_\bftheta f(\bftheta^*; \bfx, \bfv_1^M)^\intercal \right]\\
    =& \mbb{E}\left[ \sum_{i,j,p,q} \left(\frac{1}{M^2} \sum_{k, l} v_{k,i}v_{k,j}v_{l,p}v_{l,q}  \right) \left(\nabla_\bftheta \partial_i\partial_j  l_m + \frac{1}{2} \nabla_\bftheta( \partial_i l_m \partial_j l_m)\right)\left(\nabla_\bftheta \partial_p\partial_q  l_m + \frac{1}{2} \nabla_\bftheta( \partial_p l_m \partial_q l_m)\right)^\intercal \right]\\
    =& \sum_{i,j,p,q} \underbrace{\mbb{E}\left[\frac{1}{M^2} \sum_{k, l} v_{k,i}v_{k,j}v_{l,p}v_{l,q} \right]}_{\triangleq E_1} \underbrace{\mbb{E}\left[ \left(\nabla_\bftheta \partial_i\partial_j  l_m + \frac{1}{2} \nabla_\bftheta( \partial_i l_m \partial_j l_m)\right)\left(\nabla_\bftheta\partial_p\partial_q  l_m + \frac{1}{2} \nabla_\bftheta( \partial_p l_m \partial_q l_m)\right)^\intercal \right]}_{\mathfrak{V}_{ijpq}}.
\end{align*}
Continuing on $E_1$, we have that
\begin{align*}
    E_1 &= \frac{1}{M^2}\sum_{k\neq l}\mbb{E}[v_{k,i}v_{k,j}]\mbb{E}[v_{l,p}v_{l,q}] + \frac{1}{M^2}\sum_{k}\mbb{E}[v_{k,i}v_{k,j}v_{k,p}v_{k,q}]\\
    &= \bigg( 1 - \frac{1}{M}\bigg) \Sigma_{ij}\Sigma_{pq} + \frac{1}{M}\mathfrak{S}_{ijpq},
\end{align*}
which leads to \eqref{app:eqn:var}. Note that \assref{ass:v} guarantees that $|\Sigma_{ij}| < \infty$ and \assref{ass:v2} ensures $|\mathfrak{S}_{ijpq}| < \infty$.

If $p_\bfv \sim \mcal{N}(0, I)$, $\mathfrak{S}$ and $\Sigma$ have the following simpler forms
\begin{align*}
    \Sigma_{ij} &= \delta_{ij}\\
    \mathfrak{S}_{ijpq} &= \begin{cases}
        3, \quad i = j = p = q\\
        1, \quad i = j \neq p = q \text{ or } i = p \neq j = q \text{ or } i = q \neq j = p\\
        0, \quad \text{otherwise}
    \end{cases}.
\end{align*}
Then, if we assume that the second derivatives of $l_m$ are continuous, we have $\partial_i\partial_j  l_m  = \partial_j\partial_i  l_m $, and the variance \eqref{app:eqn:var} can also be simplified to
\begin{align*}
    \operatorname{Var}_{p_d, p_\bfv} \left[\nabla_\bftheta f(\bftheta^*; \bfx, \bfv_1^M) \right] = \sum_{i\neq j}V_{ij} + \frac{M+2}{M}\sum_{i} V_{ii} +  \frac{2}{M}\sum_{i\neq j}W_{ij} = \sum_{ij}V_{ij} + \frac{2}{M}\sum_{i}V_{ii} +  \frac{2}{M}\sum_{i\neq j}W_{ij}.
\end{align*}
Similarly, if $p_\bfv \sim \mcal{U}(\{\pm1\}^D)$, \eqref{app:eqn:var} has the simplified form
\begin{align*}
    \operatorname{Var}_{p_d, p_\bfv} \left[\nabla_\bftheta f(\bftheta^*; \bfx, \bfv_1^M) \right] = \sum_{ij}V_{ij} + \frac{2}{M}\sum_{i\neq j}W_{ij}.
\end{align*}
\end{proof}

\begin{customthm}{\ref{app:thm:normality}}
With the notations and assumptions in \lemref{app:lemma:derivative}, \lemref{app:lemma:variance} and \thmref{app:thm:consistency}, we have
\begin{gather*}
    \sqrt{N}(\hat{\bftheta}_{N,M} - \bftheta^*)
    \overset{d}{\to}\\
    \mcal{N}\Bigg(0, \bigg(\nabla_\bftheta^2 J(\bftheta^*; p_\bfv)\bigg)^{-1}\bigg(\sum_{i,j,p,q}\bigg[\bigg( 1 - \frac{1}{M}\bigg) \Sigma_{ij}\Sigma_{pq} + \frac{1}{M}\mathfrak{S}_{ijpq}\bigg] \mathfrak{V}_{ijpq} \bigg)\bigg(\nabla_\bftheta^2 J(\bftheta^*; p_\bfv)\bigg)^{-1} \Bigg).
\end{gather*}
In particular, if $p_\bfv \sim \mcal{N}(0, I)$, then the asymptotic variance is
\begin{align*}
    \bigg(\nabla_\bftheta^2 J(\bftheta^*)\bigg)^{-1}\bigg(\sum_{ij}V_{ij} + \frac{2}{M}\sum_{i}V_{ii} +  \frac{2}{M}\sum_{i\neq j}W_{ij}\bigg) \bigg(\nabla_\bftheta^2 J(\bftheta^*)\bigg)^{-1}.
\end{align*}
If $p_\bfv$ is the distribution of multivariate Rademacher random variables, the asymptotic variance is
\begin{align*}
    \bigg(\nabla_\bftheta^2 J(\bftheta^*)\bigg)^{-1}\bigg(\sum_{ij}V_{ij} + \frac{2}{M}\sum_{i\neq j}W_{ij}\bigg)
    \bigg(\nabla_\bftheta^2 J(\bftheta^*)\bigg)^{-1}.
\end{align*}
\end{customthm}
\begin{proof}
To simplify notations, we use $P_N h(\bfx) \triangleq \frac{1}{N}\sum_{i=1}^N h(\bfx_i, \cdot)$, where $h(\bfx, \cdot)$ is some arbitrary function. For example, $\hat{J}(\bftheta; \bfx_1^N, \bfv_{11}^{NM})$ can be written as $P_N f(\bftheta; \bfx, \bfv_1^M)$. By Taylor expansion, we can approximate $P_N \nabla_\bftheta f(\hat{\bftheta}_{N,M}; \bfx, \bfv_1^M)$ around $\bftheta^*$:
\begin{align}
    0 &= \nabla_\bftheta P_N f(\hat{\bftheta}_{N,M}; \bfx, \bfv_1^M)\notag\\
    &= P_N \nabla_\bftheta f(\bftheta^*; \bfx, \bfv_1^M) + P_N \bigg( \nabla_\bftheta^2 f(\bftheta^*; \bfx, \bfv_1^M) + E_{\hat{\bftheta}_{N,M}, \bfx, \bfv_1^M} \bigg) (\hat{\bftheta}_{N,M} - \bftheta^*),\label{app:eqn:taylor}
\end{align}
where $\|E_{\hat{\bftheta}_{N,M}, \bfx, \bfv_1^M}\|_F \leq  L(\bfx, \bfv_1^M) \|\hat{\bftheta}_{N,M} - \bftheta^*\|_2$ from \lemref{app:lemma:derivative} and Taylor expansion of vector-valued functions. Combining with the law of large numbers, we have
\begin{align*}
    P_N \nabla_\bftheta^2 f(\bftheta^*; \bfx, \bfv_1^M) = \mbb{E}_{p_d, p_\bfv}[\nabla_\bftheta^2 f(\bftheta^*; \bfx, \bfv_1^M)] + o_p(1)
\end{align*}
and
\begin{align*}
    \norm{P_N E_{\hat{\bftheta}_{N,M}}}_F \leq \mbb{E}_{p_d, p_\bfv}[L(\bfx, \bfv_1^M)]\norm{\hat{\bftheta}_{N,M} - \bftheta^*}_2 + o_p(1) = o_p(1) + o_p(1) = o_p(1),
\end{align*}
where we used $\mbb{E}[L(\bfx, \bfv_1^M)] \leq \sqrt{\mbb{E}[L^2(\bfx, \bfv_1^M)]} < \infty$ (\lemref{app:lemma:derivative}) and the consistency of $\hat{\bftheta}_{N,M}$ (\thmref{app:thm:consistency}). Now returning to \eqref{app:eqn:taylor}, we get
\begin{align*}
    &0 = P_N \nabla_\bftheta f(\bftheta^*; \bfx, \bfv_1^M) + \bigg(\mbb{E}_{p_d,p_\bfv}[\nabla_\bftheta^2 f(\bftheta^*; \bfx, \bfv_1^M)] + o_p(1) \bigg) (\hat{\bftheta}_{N,M} - \bftheta^*)\\
    \Leftrightarrow & \bigg(\nabla_\bftheta^2 J(\bftheta^*; p_\bfv) + o_p(1) \bigg)\sqrt{N} (\hat{\bftheta}_{N,M} - \bftheta^*) = -\sqrt{N} P_N \nabla_\bftheta f(\bftheta^*; \bfx, \bfv_1^M).
\end{align*}
But of course, the central limit theorem and \lemref{app:lemma:variance} yield
\begin{align*}
    -\sqrt{N} P_N \nabla_\bftheta f(\bftheta^*; \bfx, \bfv_1^M) &\overset{d}{\to} \mcal{N}(0,  \operatorname{Var}_{p_d, p_\bfv} \left[\nabla_\bftheta f(\bftheta^*; \bfx, \bfv_1^M) \right])\\
    &=  \mcal{N}\bigg(0, \sum_{i,j,p,q}\bigg[\bigg( 1 - \frac{1}{M}\bigg) \Sigma_{ij}\Sigma_{pq} + \frac{1}{M}\mathfrak{S}_{ijpq}\bigg] \mathfrak{V}_{ijpq} \bigg).
\end{align*}
Then, Slutsky's theorem gives the desired result
\begin{gather*}
    \sqrt{N}(\hat{\bftheta}_{N,M} - \bftheta^*)
    \overset{d}{\to}\\
    \mcal{N}\Bigg(0, \bigg(\nabla_\bftheta^2 J(\bftheta^*; p_\bfv)\bigg)^{-1}\bigg(\sum_{i,j,p,q}\bigg[\bigg( 1 - \frac{1}{M}\bigg) \Sigma_{ij}\Sigma_{pq} + \frac{1}{M}\mathfrak{S}_{ijpq}\bigg] \mathfrak{V}_{ijpq} \bigg)\bigg(\nabla_\bftheta^2 J(\bftheta^*; p_\bfv)\bigg)^{-1} \Bigg).
\end{gather*}
In particular, if $p_\bfv \sim \mcal{N}(0, I)$ or $p_\bfv \sim \mcal{U}(\{\pm1\}^D)$, we have $J(\bftheta^*; p_\bfv) = J(\bftheta^*)$, and therefore $\nabla_\bftheta^2 J(\bftheta^*; p_\bfv) = \nabla_\bftheta^2 J(\bftheta^*)$. We can apply \lemref{app:lemma:variance} to conclude the simplified expressions for the asymptotic variance.

\end{proof}
\begin{corollary}[Consistency and asymptotic normality of score matching]\label{app:cor:sm}
Under similar assumptions used in \thmref{app:thm:consistency} and \thmref{app:thm:normality}, we can also conclude that the score matching estimator $\hat{\bftheta}_N \triangleq \argmin_{\bftheta \in \Theta} \hat{J}(\bftheta; \bfx)$ is consistent
\begin{align*}
    \hat{\bftheta}_N \overset{p}{\to} \bftheta^*
\end{align*}
and asymptotically normal
\begin{align*}
    \sqrt{N}(\hat{\bftheta}_N - \bftheta^*) \overset{d}{\to} \mcal{N}\bigg(0, \bigg(\nabla_\bftheta^2 J(\bftheta^*)\bigg)^{-1}\bigg(\sum_{ij}V_{ij}\bigg)
    \bigg(\nabla_\bftheta^2 J(\bftheta^*)\bigg)^{-1} \bigg).
\end{align*}
\end{corollary}
\begin{proof}
    Note that 
    \begin{align*}
        \operatorname{Var}_{p_d}[\nabla_\bftheta f(\bftheta^*; \bfx)]         &= \mbb{E}_{p_d}\left[\nabla_\bftheta f(\bftheta^*; \bfx) \nabla_\bftheta f(\bftheta^*; \bfx)^\intercal  \right]\\
        &= \mbb{E}_{p_d}\left[ \left(\sum_{i} \nabla_\bftheta \partial_i\partial_i  l_m + \frac{1}{2} \nabla_\bftheta( \partial_i l_m \partial_i l_m)\right)\left(\sum_{j} \nabla_\bftheta\partial_j\partial_j  l_m + \frac{1}{2} \nabla_\bftheta( \partial_j l_m \partial_j l_m)\right)^\intercal \right]\\
         &=\sum_{ij} \mbb{E}_{p_d}\left[ \left(\nabla_\bftheta \partial_i\partial_i  l_m + \frac{1}{2} \nabla_\bftheta( \partial_i l_m \partial_i l_m)\right)\left(\nabla_\bftheta\partial_j\partial_j  l_m + \frac{1}{2} \nabla_\bftheta( \partial_j l_m \partial_j l_m)\right)^\intercal \right]\\
         &=\sum_{ij} V_{ij} 
    \end{align*}
    The other part of the proof is similar to that of \thmref{app:thm:consistency} and \thmref{app:thm:normality} and is thus obmitted.
\end{proof}

\subsection{NOISE CONTRASTIVE ESTIMATION}
\begin{proposition}\label{prop:1}
Define
\begin{align*}
  J_{\text{NCE}}(\bftheta) \triangleq  -\mbb{E}_{p_d}[\log h(\bfx;\bftheta)] - \mbb{E}_{p_n}[\log (1 - h(\bfx; \bftheta))]
\end{align*}
where 
\begin{align*}
    h(\bfx; \bftheta) &\triangleq \frac{p_m(\bfx; \bftheta) }{p_m(\bfx; \bftheta) + p_m(\bfx - \bfv; \bftheta)}\\
    p_n(\bfx) &= p_d(\bfx + \bfv).
\end{align*}
Then when $\norm{v}_2 \to 0$, we have
\begin{align*}
J_{\text{NCE}}(\bftheta) =
2\log 2 + \frac{1}{4}\mbb{E}_{p_d}\left[\bfv^\intercal \nabla^2 \log p_m(\bfx; \bftheta) \bfv + \frac{1}{2}(\nabla \log p_m(\bfx; \bftheta)^\intercal \bfv)^2\right] + o(\norm{\bfv}_2^2)
\end{align*}
\end{proposition}

\begin{proof}
    Using Taylor expansion, we can immediately get
    \begin{align*}
        \log p_m(\bfx + \bfv; \bftheta) = \log p_m(\bfx; \bftheta) + \nabla \log p_m(\bfx;\bftheta)^\intercal \bfv + \frac{1}{2} \bfv^\intercal \nabla^2 \log p_m(\bfx; \bftheta) \bfv + o(\norm{\bfv}_2^2).
    \end{align*}
    Next, observe that
    \begin{multline*}
        \log (p_m(\bfx; \bftheta) + p_m(\bfx + \bfv; \bftheta)) = \log p_m(\bfx; \bftheta) + \log \left( 1 + \exp\{\log p_m(\bfx + \bfv; \bftheta) - \log p_m(\bfx; \bftheta)\} \right)\\
        = \log p_m(\bfx; \bftheta) + \log \left(1 + \exp\left\{ \nabla \log p_m(\bfx; \bftheta)^\intercal \bfv + \frac{1}{2} \bfv^\intercal \nabla^2 \log p_m(\bfx; \bftheta) \bfv + o(\norm{\bfv}_2^2)\right\} \right)\\
        = \log p_m(\bfx; \bftheta) + \log 2 + \frac{1}{2}\left[\nabla \log p_m(\bfx; \bftheta)^\intercal \bfv + \frac{1}{2} \bfv^\intercal \nabla^2 \log p_m(\bfx; \bftheta) \bfv \right] + \frac{1}{8}(\nabla \log p_m(\bfx; \bftheta)^\intercal \bfv)^2 + o(\norm{\bfv}_2^2).
    \end{multline*}
    Similarly, we have
    \begin{multline*}
        \log (p_m(\bfx; \bftheta) + p_m(\bfx - \bfv; \bftheta))\\
        = \log p_m(\bfx; \bftheta) + \log 2 + \frac{1}{2}\left[-\nabla \log p_m(\bfx; \bftheta)^\intercal \bfv + \frac{1}{2} \bfv^\intercal \nabla^2 \log p_m(\bfx; \bftheta) \bfv \right] + \frac{1}{8}(\nabla \log p_m(\bfx; \bftheta)^\intercal \bfv)^2 + o(\norm{\bfv}_2^2).
    \end{multline*}
    Finally, note that 
    \begin{multline*}
        J_{\text{NCE}}(\bftheta) = -\mbb{E}_{p_d}[\log h(\bfx;\bftheta) + \log (1 - h(\bfx + \bfv; \bftheta))]\\
        =-\mbb{E}_{p_d}[\log p_m(\bfx; \bftheta) - \log (p_m(\bfx; \bftheta) + p_m(\bfx - \bfv; \bftheta))]
        - \mbb{E}_{p_d}[\log p_m(\bfx; \bftheta) - \log (p_m(\bfx; \bftheta) + p_m(\bfx + \bfv; \bftheta))]\\
        =2\log 2 + \frac{1}{4}\mbb{E}_{p_d}\left[\bfv^\intercal \nabla^2 \log p_m(\bfx; \bftheta) \bfv + \frac{1}{2}(\nabla \log p_m(\bfx; \bftheta)^\intercal \bfv)^2\right] + o(\norm{\bfv}_2^2),
    \end{multline*}
    as desired.
\end{proof}

\section{ADDITIONAL DETAILS OF EXPERIMENTS}
\subsection{KERNEL EXPONENTIAL FAMILIES} \label{app:dkef}
\paragraph{Model.} The kernel exponential family is a class of densities with unnormalized log density given by $\log\ptilde_f(\bfx) = f(\bfx) + \log q_0(\bfx)$. Here, $q_0$ is a fixed function and $f$ belongs to a reproducing kernel Hilbert space $\cH$, with kernel $k$ \citep{canu2006kernel, srip2017kernel}. We see this is a member of the exponential family by using reproducing property, $f(\bfx) = \langle f, k(x, \cdot)\rangle_{\cH})$. Rewriting the density, the model has natural parameter $f$ and sufficient statistic $k(\bfx, \cdot)$: 
$$\ptilde_f(\bfx) = \textrm{exp}(f(\bfx))q_0(\bfx) = \textrm{exp}(\langle f, k(x, \cdot)\rangle_{\cH}))q_0(\bfx)$$

To improve computational cost of learning $f,$ \citet{sutherland2018nystrom} use a Nystr{\"o}m-type lite approximation, selecting $L$ inducing points $\rvz_l$, and $f$ of the form:
$$f(\bfx) = \sum_{l=1}^{L}\alpha_l k(\bfx, \rvz_l)$$

We compare training using our objective against deep kernel exponential family (DKEF) models trained using exact score matching in \citet{wenliang2018}. 

The kernel $k(\bfx, \rvy)$ is a mixture of $R = 3$ Gaussian kernels, with features extracted by a neural network, $\phi_{w_r}(\cdot)$, length scales $\sigma_r$, and nonnegative mixture coefficients $\rho_r$. The neural network is a three layer fully connected network with a skip connection from the input to output layers and softplus nonlinearities. Each hidden layer has 30 neurons. We have:
$$k_w(\bfx, \rvy) = \sum_{r=1}^{R}{\rho_r \text{exp}\left(-\frac{1}{2\sigma_r^2}\left\lVert{\phi_{w_r}(\bfx) - \phi_{w_r}(\rvy)}\right\rVert^2\right)}.$$

When training DKEF models, \citet{wenliang2018} note that it is possible to analytically minimize the score matching loss over $\alpha$ because the objective is quadratic in $\alpha$. As a result, models are trained in a two step procedure: $\alpha$ is analytically minimized over a training minibatch, then the loss is computed over a validation minibatch. When training models, the analytically minimized $\alpha$ is treated as function of the other parameters. By doing this, \citet{wenliang2018} can also directly optimize the coefficient $\lambda_\bfalpha$ of a $\ell_2$ regularization loss on $\bfalpha$. This regularization coefficient is initalized to $0.01$, and is trained. (More details about the two-step optimization procedure can be found in \citet{wenliang2018}, which also includes a finalization stage for $\bfalpha$).

A similar closed form for sliced score matching, denoising score matching, approximate backpropogation, and curvature propagation can be derived. The derivation for sliced score matching is presented below for completeness.

\begin{proposition} \label{prop:alpha}
  Consider the loss
  \[
    \hat J(\bftheta, \lambda_\bfalpha; \bfx_1^N, \bfv_{11}^{NM})
    =
    \hat J(\bftheta; \bfx_1^N, \bfv_{11}^{NM})
    + \frac12
      \lambda_\bfalpha \norm{\bfalpha}_2^2
  \]
  where
  \[
    \hat J(\bftheta; \bfx_1^N, \bfv_{11}^{NM})
    =
    \frac1{N} \frac1{M} \sum_{i=1}^N  \sum_{j=1}^M  \left[
    \bfv_{ij}^\intercal \nabla^2 \log p_m(\bfx_i) \bfv_{ij} + \frac{1}{2} \left(\bfv_{ij}^\intercal \nabla \log p_m(\bfx_i) \right)^2
    \right]
  .\]
  
  For fixed $k$, $\bfz$, and $\lambda_\bfalpha$,
  as long as $\lambda_\bfalpha > 0$ then
  the optimal $\bfalpha$ is
  \begin{align*}
    \bfalpha(\lambda_\bfalpha, k, \bfz, \bfx_1^N, \bfv_{11}^{NM})
    &= \argmin_{\bfalpha} \hat J(\bftheta, \lambda_\bfalpha; \bfx_1^N, \bfv_{11}^{NM})
    = -\left(
      \mG + \lambda_\alpha \mI
    \right)^{-1} \vb
  \\
      \emG_{l,l'}
      &= \frac1{N} \frac1{M} \sum_{i=1}^N  \sum_{j=1}^M \left(\bfv_{ij}^\intercal \nabla   k(\bfx_i, \bfz_{l})\right) \left(\bfv_{ij}^\intercal \nabla   k(\bfx_i, \bfz_{l'})\right)
   \\
      \evb_{l}
      &= \frac1{N} \frac1{M} \sum_{i=1}^N  \sum_{j=1}^M \bfv_{ij}^\intercal \nabla^2 k(\bfx_i, \bfz_l) \bfv_{ij} + \left(\bfv_{ij}^\intercal \nabla \log  q_0(\bfx_i)\right) \left(\bfv_{ij}^\intercal \nabla  k(\bfx_i, \bfz_{l})\right)
  .\end{align*}
\end{proposition}
\begin{proof}
  The derivation follows very similarly to Proposition 3 in \citet{wenliang2018}. We will show that the loss is quadratic in $\bfalpha$.
  Note that
  \begin{align*}
   \frac1{N} \frac1{M} \sum_{i=1}^N  \sum_{j=1}^M  
    \bfv_{ij}^\intercal \nabla^2 \log p_m(\bfx_i) \bfv_{ij} 
  &=    \frac1{N} \frac1{M} \sum_{i=1}^N  \sum_{j=1}^M  \left[\sum_{m=1}^M
    \alpha_l\bfv_{ij}^\intercal \nabla^2 k(\bfx_i, \bfz_l) \bfv_{ij} \right] + \up{C}
    \\&= \bfalpha\tp \left[
         \frac1{N} \frac1{M} \sum_{i=1}^N  \sum_{j=1}^M \bfv_{ij}^\intercal \nabla^2 k(\bfx_i, \bfz_l) \bfv_{ij}
       \right]_l
       + \up{C}
    \\   \frac1{N} \frac1{M} \sum_{i=1}^N  \sum_{j=1}^M 
    \frac12 \left(\bfv_{ij}^\intercal \nabla \log p_m(\bfx_i) \right)^2
    &= \frac1{N} \frac1{M} \sum_{i=1}^N  \sum_{j=1}^M  \frac12 \left(
         \sum_{m, m'=1}^M \alpha_l \alpha_{m'} \left(\bfv_{ij}^\intercal \nabla   k(\bfx_i, \bfz_{l})\right) \left(\bfv_{ij}^\intercal \nabla   k(\bfx_i, \bfz_{l'})\right)
\right.\\&\qquad\left.
       + 2 \sum_{m=1}^M \alpha_l \left(\bfv_{ij}^\intercal \nabla \log  q_0(\bfx_i)\right) \left(\bfv_{ij}^\intercal \nabla  k(\bfx_i, \bfz_{l})\right)
       + \left(\bfv_{ij}^\intercal \nabla \log  q_0(\bfx_i)\right)^2
       \right)
  \\&= \frac12 \bfalpha\tp \mG \bfalpha + \bfalpha\tp \left[ \frac1{N} \frac1{M} \sum_{i=1}^N  \sum_{j=1}^M \left(\bfv_{ij}^\intercal \nabla \log  q_0(\bfx_i)\right) \left(\bfv_{ij}^\intercal \nabla  k(\bfx_i, \bfz_{l})\right) \right] + \up{C}
  .\end{align*}

  Thus the overall optimization problem is
  \begin{align*}
       \bfalpha(\lambda_\bfalpha, k, \bfz, \bfx_1^N, \bfv_{11}^{NM})
    &= \argmin_{\bfalpha} \hat J(\bftheta, \lambda_\bfalpha; \bfx_1^N, \bfv_{11}^{NM})
  \\&= \argmin_{\bfalpha} \frac12 \bfalpha\tp \left( \mG + \lambda_\bfalpha \mI \right) \bfalpha + \bfalpha\tp \vb
  .\end{align*}
  Because $\lambda_\alpha > 0$ and $\mG$ is positive semidefinite,
  the matrix in parentheses is strictly positive definite,
  and the claimed result follows directly from standard vector calculus.
\end{proof}

\paragraph{Hyperparameters.}

RedWine and WhiteWine are dequantized by adding 
uniform noise to each dimension in the range $[-d, d]$ where $d$ is the median distance between two values for that dimension. For each dataset, 10\% of the entire data was used as testing, and 10\% of the remaining was used for
validation. PCA whitening is applied to the data. Noise of standard deviation 0.05 is added as part of preprocessing. 

The DKEF models have $R = 3$ Gaussian kernels. Each feature extractor is a 3-layer neural network with a skip connection from the input to output, with 30 hidden neurons per layer. Weights were initialized from a Gaussian distribution with standard
deviation equal to $\frac{1}{\sqrt{30}}$. Length scales $\sigma_r$ were initialized to 1.0, 3.3 and 10.0. We use $L = 200$ trainable inducing points, which were initialized from training data. 

Models are trained using an Adam optimizer, with learning rate $10^{-2}$. A batch size of 200 is used, with 100 points for computing $\bfalpha$, and 100 for computing the loss. Models are stopped after validation loss does not improve for 200 steps. 

For denoising score matching, we perform a grid search with values [0.02, 0.04, 0.06, 0.08, 0.10, 0.12, 0.14, 0.16, 0.20, 0.24, 0.28, 0.32, 0.40, 0.48, 0.56, 0.64, 1.28]. We train models for each value of $\sigma$ using two random seeds, and pick the $\sigma$ with the best average validation score matching loss. For curvature propagation, one noise sample is used to match the performance of sliced score matching. 

\paragraph{Log-likelihoods.} Log-likelihoods are presented below. They are estimated using AIS, using a proposal distribution $\mcal{N}(0, 2I)$, using 1,000,000 samples. 

\begin{figure*}[h]
    \begin{center}
    \includegraphics[width=0.9\textwidth]{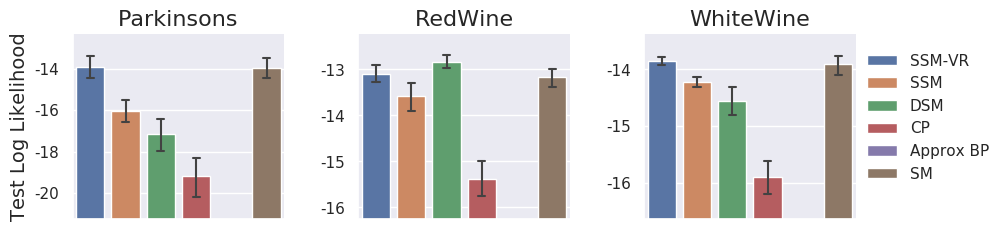}
    \end{center}
    \caption{Log-likelihoods after training DKEF models on UCI datasets with different loss functions; higher is better. Results for approximate backpropagation are not shown because log-likelihoods were smaller than $-10^{6}$.}
    \label{fig:dkef}
\end{figure*}

\subsection{NICE} \label{app:nice}

\paragraph{Hyperparameters and Model Architecture.}
The model has four coupling layers, each with five hidden layers, for a total of 20 hidden layers, as well as a final scale layer \citep{dinh14nice}. Softplus nonlinearities are used between hidden layers. 

Models are trained using the Adam optimizer with learning rate $10^{-3}$ for 100 epochs. The best checkpoint on exact score matching loss, evaluation every 100 iterations, is used to report test set performance. We use a batch size of 128.

Data are dequantized by adding uniform noise in the range $[-\frac{1}{512}, \frac{1}{512}]$, clipped to be in the range $[-0.001, 0.001]$, and then transformed using a logit transformation $\log(x) - \log(1-x)$. 90\% of the training set is used for training, and 10\% for validation, and the standard test set is used. 

For grid search for the optimal value of $\sigma$, eight values are used: [0.01, 0.05, 0.10, 0.20, 0.28, 0.50, 1.00, 1.50]. We also evaluate $\sigma = 1.74$, chosen by the heuristic in \citet{saremi2018deep}. The model with the best performance on validation score matching loss is used. Only nine values of $\sigma$ are evaluated because training each model takes approximately two hours.

\subsection{SCORE ESTIMATION FOR IMPLICIT DISTRIBUTIONS}\label{app:sec:vi}\label{app:sec:wae}
\paragraph{Architectures.} We put the architectures of all networks used in the MNIST and CelebA experiments in \tabref{tab:mnist_arch} and \tabref{tab:celeba_arch} respectively. 
\paragraph{Training.} For MNIST experiments, we use RMSProp optimizer with a learning rate of 0.001 for all methods. On CelebA, the learning rate is changed to 0.0001. All algorithms are trained for 100000 iterations with a batch size of 128.
\paragraph{Samples.} All samples are generated after 100000 training iterations.
\begin{table}
	\small
	\centering
	\begin{tabular}{l|c|c}
		\Xhline{3\arrayrulewidth}\bigstrut
		\bf Name & \bf Configuration & \bf Algorithm \\
		\hline\bigstrut
		\multirow{3}{*}{Encoder} & Linear(784, 256), Tanh & \multirow{3}{*}{ELBO VAE, WAE} \\
		& Linear(256, 256), Tanh & \\
		& Linear(256, $D_\bfz$) & \\
		\hline\bigstrut
		\multirow{3}{*}{Implicit Encoder} & Linear(784 + $D_\bfe$, 256), Tanh & \multirow{3}{*}{Implicit VAE} \\
		& Linear(256, 256), Tanh & \\
		& Linear(256, $D_\bfz$) & \\
		\hline\bigstrut
		\multirow{3}{*}{Decoder} & Linear($D_\bfz$, 256), Tanh & \multirow{3}{*}{All} \\
		& Linear(256, 256), Tanh & \\
		& Linear(256, 784), Sigmoid & \\
		\hline\bigstrut
		\multirow{3}{*}{Score Estimator} & Linear(784 + $D_\bfz$, 256), Tanh & \multirow{3}{*}{Implicit VAE} \\
		& Linear(256, 256), Tanh & \\
		& Linear(256, $D_\bfz$) & \\
		\hline\bigstrut
		\multirow{3}{*}{Score Estimator} & Linear($D_\bfz$, 256), Tanh & \multirow{3}{*}{WAE} \\
		& Linear(256, 256), Tanh & \\
		& Linear(256, $D_\bfz$) & \\
		\Xhline{3\arrayrulewidth}
	\end{tabular}
	\caption{Architectures on MNIST. In our models, $D_\bfe = D_\bfz$. $D_\bfz$ takes the values 8 and 32 in different experiments.}\label{tab:mnist_arch}
\end{table}

\begin{table}[h]
	\small
	\centering
	\begin{tabular}{l|c|c}
		\Xhline{3\arrayrulewidth}\bigstrut
		\bf Name & \bf Configuration & \bf Algorithm \\
		\hline\bigstrut
		\multirow{6}{*}{Encoder} & $5 \times 5$ conv; $m$ maps; stride $2 \times 2$; padding $2$, ReLU & \multirow{6}{*}{ELBO VAE, WAE} \\
		& $5 \times 5$ conv; $2m$ maps; stride $2 \times 2$; padding $2$, ReLU & \\
		& $5 \times 5$ conv; $4m$ maps; stride $2 \times 2$; padding $2$, ReLU & \\
		& $5 \times 5$ conv; $8m$ maps; stride $2 \times 2$; padding $2$, ReLU & \\
		& 512 Dense, ReLU & \\
		& $D_\bfz$ Dense & \\
		\hline\bigstrut
		\multirow{7}{*}{Implicit Encoder} & concat $[\bfx,$ ReLU(Dense($\bfe$))$]$ along channels & \multirow{7}{*}{Implicit VAE} \\
		& $5 \times 5$ conv; $m$ maps; stride $2 \times 2$; padding $2$, ReLU & \\
		& $5 \times 5$ conv; $2m$ maps; stride $2 \times 2$; padding $2$, ReLU & \\
		& $5 \times 5$ conv; $4m$ maps; stride $2 \times 2$; padding $2$, ReLU & \\
		& $5 \times 5$ conv; $8m$ maps; stride $2 \times 2$; padding $2$, ReLU & \\
		& 512 Dense, ReLU & \\
		& $D_\bfz$ Dense & \\
		\hline\bigstrut
		\multirow{6}{*}{Decoder} & Dense, ReLU & \multirow{6}{*}{All} \\
		& $5 \times 5$ $\textrm{conv}^\intercal$; $4m$ maps; stride $2 \times 2$; padding $2$; out padding $1$, ReLU & \\
		& $5 \times 5$ $\textrm{conv}^\intercal$; $2m$ maps; stride $2 \times 2$; padding $2$; out padding $1$, ReLU & \\
		& $5 \times 5$ $\textrm{conv}^\intercal$; $1m$ maps; stride $2 \times 2$; padding $2$; out padding $1$, ReLU & \\
		& $5 \times 5$ $\textrm{conv}^\intercal$; $c$ maps; stride $2 \times 2$; padding $2$; out padding $1$, Tanh & \\
		\hline\bigstrut
		\multirow{7}{*}{Score Estimator} & concat $[\bfx,$ ReLU(Dense($\bfz$))$]$ along channels & \multirow{7}{*}{Implicit VAE} \\
		& $5 \times 5$ conv; $m$ maps; stride $2 \times 2$; padding $2$, ReLU & \\
		& $5 \times 5$ conv; $2m$ maps; stride $2 \times 2$; padding $2$, ReLU & \\
		& $5 \times 5$ conv; $4m$ maps; stride $2 \times 2$; padding $2$, ReLU & \\
		& $5 \times 5$ conv; $8m$ maps; stride $2 \times 2$; padding $2$, ReLU & \\
		& 512 Dense, ReLU & \\
		& $D_\bfz$ Dense & \\
		\hline\bigstrut
		\multirow{7}{*}{Score Estimator} & Reshape(ReLU(Dense($\bfz$))$]$) to 1 channel & \multirow{7}{*}{WAE} \\
		& $5 \times 5$ conv; $m$ maps; stride $2 \times 2$; padding $2$, ReLU & \\
		& $5 \times 5$ conv; $2m$ maps; stride $2 \times 2$; padding $2$, ReLU & \\
		& $5 \times 5$ conv; $4m$ maps; stride $2 \times 2$; padding $2$, ReLU & \\
		& $5 \times 5$ conv; $8m$ maps; stride $2 \times 2$; padding $2$, ReLU & \\
		& 512 Dense, ReLU & \\
		& $D_\bfz$ Dense & \\
		\Xhline{3\arrayrulewidth}
	\end{tabular}
	\caption{Architectures on CelebA. In our models, $D_\bfe = D_\bfz$. $D_\bfz$ takes the values 8 or 32 in different experiments.}\label{tab:celeba_arch}
\end{table}

\section{VARIANCE REDUCTION}\label{app:sec:vr}
Below we discuss approaches to reduce the variance of $\hat{J}(\bftheta; \bfx_1^N, \bfv_{11}^{NM})$, which can lead to better performance in practice. The most na\"{i}ve approach, of course, is using a larger $M$ to compute $\hat{J}(\bftheta; \bfx_1^N, \bfv_{11}^{NM})$. However, this requires more computation and when $M$ is close to the data dimension, sliced score matching will lose its computational advantage over score matching. 

An alternative approach is to leverage control variates~\citep{mcbook}. A control variate is a random variable whose expectation is tractable, and is highly correlated with another random variable without a tractable expectation. Define $c(\bftheta; \bfx, \bfv) \triangleq \frac{1}{2}(\bfv^\intercal \bfs_m(\bfx;\bftheta))^2$. Note that when $p_\bfv$ is a multivariate standard normal or multivariate Rademacher distribution, $c(\bftheta; \bfx, \bfv)$ will have a tractable expectation, \ie,
\begin{align*}
    \mbb{E}_{p_\bfv}[c(\bftheta; \bfx, \bfv)] = \frac{1}{2}\norm{\bfs_m(\bfx;\bftheta)}^2_2,
\end{align*}
which is easily computable. Now let $\beta(\bfx) c(\bftheta; \bfx, \bfv)$ be our control variate, where $\beta(\bfx)$ is a function to be determined. Due to the structural similarity between $\beta(\bfx) c(\bftheta; \bfx, \bfv)$ and $\hat{J}(\bftheta; \bfx_1^N, \bfv_{11}^{NM})$, it is easy to believe that $\beta(\bfx) c(\bftheta; \bfx, \bfv)$ can be a correlated control variate with an appropriate $\beta(\bfx)$. We thus consider the following objective
\begin{align*}
    \hat{J}_{\text{vr}}(\bftheta; \bfx_1^N, \bfv_{11}^{NM}) \triangleq \hat{J}(\bftheta; \bfx_1^N, \bfv_{11}^{NM}) - \frac{1}{N}\sum_{i=1}^N \beta(\bfx_i)\bigg( \frac{1}{M}\sum_{j=1}^M c(\bftheta; \bfx_i, \bfv_{ij}) - \frac{1}{2}\norm{\bfs_m(\bfx_i; \bftheta)}_2^2\bigg).
\end{align*}
Note that $\mbb{E}[\hat{J}_{\text{vr}}(\bftheta; \bfx_1^N, \bfv_{11}^{NM})] = J(\bftheta; p_\bfv)$. The theory of control variates guarantees the existence of $\beta(\bfx)$ that can reduce the variance. In practice, there can be many heuristics of choosing $\beta(\bfx)$, and we found that $\beta(\bfx) \equiv 1$ can often be a good choice in our experiments.

\newpage
\section{PSEUDOCODE}
\begin{algorithm}
	\caption{Score Matching}
	\label{alg:sm}
	\begin{algorithmic}[1]
	    \Require{$\tilde{p}_m(\cdot; \bftheta), \bfx$}
        \State{$\bfs_m(\bfx;\bftheta) \gets \texttt{grad}(\log \tilde{p}_m(\bfx; \bftheta), \bfx)$}
        \State{$J \gets \frac12\norm{\bfs_m(\bfx;\bftheta)}_2^2$}
        \For{$d \gets 1$ to $D$}
		\Comment{For each diagonal entry}
        \State{$(\nabla_\bfx\bfs_m(\bfx;\bftheta))_d \gets \texttt{grad}((\bfs_m(\bfx; \bftheta))_d, \bfx)_d$}
        \State{$J \gets J + (\nabla_\bfx\bfs_m(\bfx;\bftheta))_d$}
        \EndFor
        \item[]
        \Return{$J$}
	\end{algorithmic}
\end{algorithm}

\end{document}